\documentclass[letterpaper]{article} 
\usepackage{aaai2026}  
\usepackage{times}  
\usepackage{helvet}  
\usepackage{courier}  
\usepackage[hyphens]{url}  
\usepackage{graphicx} 
\usepackage{natbib}  
\usepackage{caption} 
\usepackage{algorithm}
\usepackage{algorithmic}

\usepackage{subcaption}
\usepackage{booktabs}

\usepackage{amsmath}
\usepackage{amssymb}
\usepackage{multirow}

\usepackage{booktabs}  
\usepackage{amsthm}
\newtheorem{lemma}{Lemma} 
\newtheorem{theorem}{Theorem}

\newtheorem{definition}{Definition}
\newtheorem{proposition}{Proposition}

\usepackage{pifont}

\usepackage[table]{xcolor}  
\usepackage{tabularx} 
\usepackage{newfloat}
\usepackage{listings}
\DeclareCaptionStyle{ruled}{labelfont=normalfont,labelsep=colon,strut=off} 
\floatstyle{ruled}
\newfloat{listing}{tb}{lst}{}
\floatname{listing}{Listing}
\title{From Points to Coalitions: Hierarchical Contrastive Shapley Values for Prioritizing Data Samples}

\author {
    Canran Xiao\textsuperscript{\rm 1},
    Jiabao Dou\textsuperscript{\rm 2},
    Zhiming Lin\textsuperscript{\rm 3}\thanks{Corresponding author.},
    Zong Ke\textsuperscript{\rm 4},
    Liwei Hou\textsuperscript{\rm 5}
}
\affiliations {
    \textsuperscript{\rm 1}School of Cyber Science and Technology, Shenzhen Campus of Sun Yat-sen University, Shenzhen, China\\
    \textsuperscript{\rm 2}Department of Computer Science, Hong Kong Baptist University, Hong Kong\\
    \textsuperscript{\rm 3}School of Business, Nankai University, Tianjin, China\\
    \textsuperscript{\rm 4}Faculty of Science, National University of Singapore, Singapore\\
    \textsuperscript{\rm 5}School of Artificial Intelligence and Robotics, Hunan University, Changsha, China\\
    xiaocanran999@gmail.com, 22258248@life.hkbu.edu.hk, nklinzhiming@gmail.com, a0129009@u.nus.edu, houliwei@hnu.edu.cn
}

\usepackage{bibentry}

\begin{document}

\maketitle

\begin{abstract}
How should we quantify the value of each training example when datasets are large, heterogeneous, and geometrically structured?  
Classical Data‑Shapley answers in principle, but its $O(n!)$ complexity and point‑wise perspective are ill‑suited to modern scales.  
We propose \emph{Hierarchical Contrastive Data Valuation} (\textsc{HCDV}), a three‑stage framework that  
(\textit{i}) learns a contrastive, geometry‑preserving representation,  
(\textit{ii}) organises the data into a balanced coarse‑to‑fine hierarchy of clusters, and  
(\textit{iii}) assigns Shapley‑style payoffs to coalitions via local Monte‑Carlo games whose budgets are propagated downward.  
\textsc{HCDV} collapses the factorial burden to $O\!\bigl(T\sum_{\ell}K_{\ell}\bigr)=O(TK_{\max}\log n)$, rewards examples that sharpen decision boundaries, and regularises outliers through curvature‑based smoothness.  
We prove that \textsc{HCDV} approximately satisfies the four Shapley axioms with surplus loss $O(\eta\log n)$, enjoys sub‑Gaussian coalition deviation $\tilde{O}(1/\sqrt{T})$, and incurs at most $k\varepsilon_\infty$ regret for top‑$k$ selection.  
Experiments on four benchmarks—tabular, vision, streaming, and a 45M‑sample CTR task—plus the \textsc{OpenDataVal} suite show that \textsc{HCDV} lifts accuracy by up to \(+5\)pp, slashes valuation time by up to \(100\times\), and directly supports tasks such as augmentation filtering, low‑latency streaming updates, and fair marketplace payouts.
\end{abstract}

\section{Introduction}
\label{sec:intro}
Data valuation~\citep{sim2022data,wang2023data,bendechache2023systematic,chen2025framework} plays a pivotal role in modern machine learning pipelines~\citep{shen2024efficient}. As data becomes massive and heterogeneous, quantifying the importance of individual data points---or groups of data points---helps practitioners in data curation~\citep{bhardwaj2024machine,andrews2024ethical,wang2024computing}, active sampling~\citep{yao2024swift,xu2021validation,goetz2019active}, federated learning~\citep{wang2020principled,fan2022improving,li2024data,xiao2024confusion}, and fair data pricing~\citep{pei2020survey,zhang2023survey}. 
Its significance further extends to a wide spectrum of modern applications, including autonomous systems~\citep{yao2023ndc,zhang2024cf,zhang2023multi,jiang2025transforming,xiao2025diffusion,xiao2025multifrequency}, intelligent healthcare~\citep{tong2025progress,liu2025memeblip2,wang2025medical}, and recommender systems~\citep{zhao2025generative}.

Formally, let $\mathcal{D} = \{(x_i, y_i)\}_{i=1}^n$ be a dataset of $n$ examples, with $x_i \in \mathcal{X}$ representing features (possibly high-dimensional) and $y_i \in \mathcal{Y}$ being associated labels. We denote by $v(S)$ a performance function measuring the quality (e.g., accuracy, negative loss) of a model trained on a subset $S \subseteq \mathcal{D}$. A data valuation function then assigns a numerical score $\phi_i$ to each data point $(x_i, y_i)$, reflecting its contribution to $v(\cdot)$ when considering all possible subsets.

A powerful theoretical basis for data valuation is derived from the Shapley value \cite{hart1989shapley}(SV), which, for each point $i$, is given by:
\begin{equation}
\label{eq:shapley_value_intro}
    \phi_i(\mathcal{D}) \;=\; \sum_{S \subseteq \mathcal{D}\setminus\{i\}} 
    \frac{|S|!(\,n - |S| - 1)!}{n!} 
    \Bigl[v(S \cup \{i\}) - v(S)\Bigr].
\end{equation}
Eq.~\eqref{eq:shapley_value_intro} offers a principled way to distribute the overall ``value'' of a dataset among its points, satisfying fair attribution axioms
.

Despite its elegance, applying the Shapley formula directly in practical scenarios encounters two key problems: 
(i) \emph{Combinatorial explosion.}  
    Exact computation is $O(n!)$; even Monte-Carlo estimators can be prohibitive for large $n$~\citep{fleckenstein2023review,xu2021validation}.  
(ii) \emph{Structural myopia.}  
    Treating every record as an isolated ``player'' ignores latent geometry-manifolds, semantic clusters, and causal strata-that actually govern generalisation~\citep{whang2023data}.  
\begin{quote}
\emph{How much is one example worth \textbf{when its neighbours speak on its behalf}?  
And if we let the geometry of the data-not just the individual points-join the conversation, would our notion of ``value'' change?}
\end{quote}

These questions motivate a fresh perspective: re-design the game itself so that the players are multiscale, geometry-aware neighbourhoods rather than isolated points.  
Doing so leads to our \textbf{Hierarchical Contrastive Data Valuation (\textsc{HCDV})}, whose key ideas are:
\emph{we treat a dataset not as a crowd of isolated points but as a community of neighbourhoods that talk to one another}. 

Within this multiscale framework, valuation proceeds in a coarse-to-fine manner: higher-level coalitions first apportion their collective utility, which is then recursively distributed to their constituent sub-coalitions.  
This hierarchical decomposition (i) mitigates factorial complexity by confining the combinatorial search to a modest number of clusters per level;  
(ii) accentuates informational distinctiveness by awarding greater utility to coalitions that sharpen geometric boundaries in the learned representation space; and  
(iii) preserves robustness and interpretability through smoothness regularisation, preventing outliers from exerting disproportionate influence.

The main contributions are as follows:
(\textbf{i}) We formulate Hierarchical Contrastive Data Valuation (HCDV), a scalable, geometry-aware alternative to classical Shapley data valuation.
(\textbf{ii}) We provide theoretical guarantees that \textsc{HCDV} approximates Shapley's efficiency, symmetry, dummy, and additivity axioms, with approximation error linked to cluster granularity.
(\textbf{iii}) We demonstrate, across tabular, vision, and streaming benchmarks, that \textsc{HCDV} uncovers hidden synergies, supports more effective active sampling and data pricing, and reduces runtime by one to two orders of magnitude compared with state-of-the-art Shapley approximators.

\section{Preliminaries}
\label{sec:prelim}
\subsection{Data Valuation Framework}

Consider a supervised learning setting with an input space $\mathcal{X}$ and a label space $\mathcal{Y}$.  
Let the training dataset be
$
    \mathcal{D} \;=\; \bigl\{(x_i,y_i)\bigr\}_{i=1}^{n},
    (x_i,y_i) \in \mathcal{X}\times\mathcal{Y}.
$
For any subset $S \subseteq \mathcal{D}$, a training operator
$
    \mathcal{T} : 2^{\mathcal{D}} \longrightarrow \mathcal{H}
$
maps $S$ to a model $M_S = \mathcal{T}(S)$ in a hypothesis space $\mathcal{H}$ (e.g., the parameter space of a neural network).  
Model quality is assessed on a fixed validation set $\mathcal{D}_{\mathrm{val}}$ using metric  
$\mathcal{M}\!:\!\mathcal{H}\!\times\!2^{\mathcal{D}} \rightarrow \mathbb{R}$  
such as accuracy, balanced accuracy, or negative loss.  
We shorthand the induced characteristic function:
\begin{equation}
    v(S)
    \;:=\;
    \mathcal{M}\bigl(\mathcal{T}(S),\;\mathcal{D}_{\mathrm{val}}\bigr),
    \qquad
    S \subseteq \mathcal{D},
\end{equation}
which plays the role of a \emph{payoff} in cooperative-game terminology.

\begin{definition}[Data Valuation Function]
A \emph{data valuation function} is a mapping  
$\phi : \mathcal{D} \rightarrow \mathbb{R}$  
that assigns to each point $(x_i,y_i)$ a real-valued score $\phi_i$.  
Intuitively, $\phi_i$ quantifies the marginal performance gain attributable to $(x_i,y_i)$ when all possible coalitions of data are taken into account.
\end{definition}

\subsection{Canonical Axioms for Data Valuation}
Let $\phi = (\phi_1,\dots,\phi_n)$ denote the valuation vector produced by a scheme under characteristic function~$v(\cdot)$.  

\paragraph{Efficiency (Completeness).}The total assigned value equals the global performance surplus obtained by using the full dataset versus no data.
\begin{equation}
    \sum_{i=1}^{n} \phi_i
    \;=\;
    v(\mathcal{D}) \;-\; v(\varnothing).
\end{equation}

\paragraph{Symmetry (Fairness).}
If two data points $i$ and $j$ are \emph{indistinguishable}-that is,
\begin{equation}
    v\bigl(S\!\cup\!\{i\}\bigr) \;=\; v\bigl(S\!\cup\!\{j\}\bigr),
    \quad
    \forall\,S\subseteq\mathcal{D}\setminus\{i,j\},
\end{equation}
then their valuations coincide:
\begin{equation}
    \phi_i = \phi_j.
\end{equation}

\paragraph{Dummy Player.}
If a point $k$ never changes the performance of any coalition,
\begin{equation}
    v\bigl(S\!\cup\!\{k\}\bigr) \;=\; v(S),
    \quad
    \forall\,S\subseteq\mathcal{D},
\end{equation}
then $\phi_k = 0$.

\paragraph{Additivity.}
Given two characteristic functions $v_1$ and $v_2$ defined on the same dataset, let $v_1 + v_2$ denote their pointwise sum.  
A valuation scheme is \emph{additive} if
\begin{equation}
    \phi_i^{(v_1+v_2)}
    \;=\;
    \phi_i^{(v_1)} \;+\; \phi_i^{(v_2)},
    \quad
    \forall\,i\in\{1,\dots,n\}.
\end{equation}
Additivity ensures that the value assigned under multiple, simultaneously considered payoffs is the linear superposition of the values computed for each payoff separately.

\section{Hierarchical Contrastive Data Valuation}
\label{sec:method}

This section formalises HCDV, a three-stage procedure that (i) learns a geometry-preserving representation; (ii) organises the data into a coarse-to-fine hierarchy of coalitions; and (iii) computes Shapley-style payoffs for those coalitions under a contrastive characteristic function, ultimately yielding a valuation $\phi_i$ for every data point.

\subsection{Stage\,I: Geometry-Preserving Representation}
\label{sec:method:embedding}

\paragraph{Embedding model.}
Let $f_\theta:\mathcal{X}\!\to\!\mathbb{R}^d$ be a neural encoder with parameters~$\theta$.  
For any subset $S\!\subseteq\!\mathcal{D}$, define the \emph{base utility}
\begin{equation}
    \mathcal{M}(S)
    \;:=\;
    \mathcal{M}\bigl(\mathcal{T}(S),\mathcal{D}_{\mathrm{val}}\bigr)
    \in[0,1],
\end{equation}
and the \emph{contrastive dispersion}
\begin{equation}
\label{eq:contrast_dispersion}
    \Delta_{\!\mathrm{c}}(S)
    \;:=\;
    \sum_{(i,j)\in\mathcal{P}(S)}
        d\!\bigl(f_\theta(x_i),f_\theta(x_j)\bigr),
\end{equation}
where $\mathcal{P}(S)$ contains all unordered pairs whose labels differ, and
$d(\cdot,\cdot)$ is a metric in~$\mathbb{R}^d$ (cosine distance in experiments).

\paragraph{Embedding objective.}
We obtain $\theta^\star$ by maximising
\begin{equation}
\label{eq:embed_objective}
    \max_{\theta}
    \;\Bigl\{
        \mathbb{E}_{S\sim\mathcal{P}_\text{batch}}
            \bigl[\mathcal{M}(S)+\lambda\,\Delta_{\!\mathrm{c}}(S)\bigr]
        \;-\;
        \alpha\,\Omega(\theta)
    \Bigr\},
\end{equation}
where $\Omega(\theta)
    := \sum_{p,q}\|\nabla_{x_p}\,d(f_\theta(x_p),f_\theta(x_q))\|_2^2$
is a smoothness regulariser,  
$\lambda\!>\!0$ balances discrimination and accuracy,  
$\alpha\!>\!0$ controls smoothness,  
and $\mathcal{P}_\text{batch}$ samples mini-batches for contrastive optimisation.

\subsection{Stage\,II: Hierarchical Decomposition}
\label{sec:method:hierarchy}

\paragraph{Recursive clustering.}
Using $z_i:=f_{\theta^\star}(x_i)$, we recursively partition $\mathcal{D}$:

\begin{equation}
    \mathcal{D}=C_0
      \;\;\xrightarrow{\text{split into }K_1}\;\;
      C_1
      \;\;\xrightarrow{\text{split}}\;\;
      \cdots
      \;\;\xrightarrow{\text{split}}\;\;
      C_L,
\end{equation}
where $C_\ell=\{G_1^{(\ell)},\dots,G_{K_\ell}^{(\ell)}\}$ and  
$|G_k^{(L)}|\le M$ (a user-chosen leaf size).  
We use balanced $k$-means with $k=K_\ell$ at depth~$\ell$ by default;  
any deterministic or stochastic clustering is admissible.

\subsection{Stage\,III: Multi‑Resolution Shapley Attribution}
\label{sec:method:valuation}

Given the hierarchy $\{C_\ell\}_{\ell=0}^{L}$, we attribute value in a
\textit{top–down} fashion.
At every depth $\ell$ we run a \emph{local} cooperative game whose players
are the $K_\ell$ coalitions
$C_\ell=\{G^{(\ell)}_1,\dots,G^{(\ell)}_{K_\ell}\}$.
The procedure for one level consists of two tightly‑coupled steps.

\paragraph{Local Shapley estimation.}
Each coalition’s marginal utility is measured under the characteristic
function
\begin{equation}
    v_\ell(S)
    :=
    \mathcal{M}\!\bigl(\textstyle\bigcup_{G\in S}G\bigr)
    +
    \lambda\,\Delta_{\!c}(S),
    \qquad
    S\subseteq C_\ell,
\end{equation}
where $\Delta_{\!c}$ is defined in Eq.\,\eqref{eq:contrast_dispersion}.
Because $K_\ell$ seldom exceeds $\mathcal{O}(10^{1})$, the Shapley
value of $G^{(\ell)}_i$ can be estimated accurately with $T$ random
permutations:
\begin{equation} \label{eq:111}
    \widehat{\psi}^{(\ell)}_i
    =
    \frac{1}{T}\sum_{t=1}^{T}
        \Bigl[
           v_\ell\!\bigl(\mathrm{Pre}_{\pi_t}(G^{(\ell)}_i)\cup\{G^{(\ell)}_i\}\bigr)
           -
           v_\ell\!\bigl(\mathrm{Pre}_{\pi_t}(G^{(\ell)}_i)\bigr)
        \Bigr],
\end{equation}
with $\pi_t\!\sim\!\text{Unif}(\mathfrak{S}_{K_\ell})$ and
$\mathrm{Pre}_{\pi_t}$ denoting predecessors in~$\pi_t$.
Proposition~\ref{prop:concentration} shows that
$\|\widehat{\psi}^{(\ell)}-\psi^{(\ell)}\|_\infty
= \mathcal{O}_{\mathbb{P}}\!\bigl(B\sqrt{\tfrac{\log K_\ell}{T}}\bigr)$,
so $T\!=\!256$ suffices in practice.

\paragraph{Budget down‑propagation.}
The scalar $\widehat{\psi}^{(\ell)}_i$ represents the \emph{total} credit
earned by coalition $G^{(\ell)}_i$ at level $\ell$.
To refine this credit among its $m$ child coalitions
$\{G^{(\ell+1)}_{i1},\dots,G^{(\ell+1)}_{im}\}$ we compute
non‑negative weights
\begin{equation}
    \omega_{ij}^{(\ell+1)}
    :=
    \frac{\max\{\,v_\ell(\{G^{(\ell+1)}_{ij}\}),\,0\}}
         {\sum_{j'=1}^{m}\max\{\,v_\ell(\{G^{(\ell+1)}_{ij'}\}),\,0\}},
\end{equation}
and allocate
\begin{equation} \label{eq:234}
    \widetilde{\psi}^{(\ell+1)}_{ij}
    \;=\;
    \omega_{ij}^{(\ell+1)}\,
    \widehat{\psi}^{(\ell)}_i .
\end{equation}
Eq.~\eqref{eq:234} conserves mass:
$\sum_{j}\widetilde{\psi}^{(\ell+1)}_{ij}=\widehat{\psi}^{(\ell)}_i$, so the
global efficiency deviation in Theorem~\ref{thm:efficiency} grows only
linearly with depth.
Crucially, the propagated budgets merely cap the \emph{pot} available at
depth $\ell\!+\!1$; the children still play their \emph{own} Shapley game
with characteristic function $v_{\ell+1}(\cdot)$, allowing us to capture
interactions that are invisible at coarser resolutions.

\paragraph{Leaf valuation.}
The recursion stops at depth $L$ where every coalition contains at most
$M$ points.
Because $M$ is user‑controlled, we either
evaluate exact Shapley among the at‑most‑$M$ players or, when desired,
divide the residual budget uniformly.
The resulting vector
$\phi=(\phi_1,\dots,\phi_n)$ satisfies
$\sum_i \phi_i = v(\mathcal{D})-v(\varnothing)\pm\mathcal{O}(L\delta)$,
with $\delta$ defined in Theorem~\ref{thm:efficiency},
while incurring a total cost of
$\mathcal{O}(T\sum_{\ell=0}^{L}K_\ell)$—several orders below the
$\mathcal{O}(n!)$ complexity of flat Shapley computation.

\subsection{Algorithmic Summary}
The \textsc{HCDV} algorithm is described in Algorithm \ref{alg:hcdv}.

\begin{algorithm}[ht]
\caption{\textsc{HCDV}}
\label{alg:hcdv}
\begin{algorithmic}[1]
\REQUIRE Dataset $\mathcal{D}$; hierarchy depth $L$; cluster counts $\{K_\ell\}_{\ell=1}^{L}$; leaf size $M$; hyper‑parameters $\lambda,\alpha$; permutation budget $T$
\ENSURE Point‑level valuations $\{\phi_i\}_{i=1}^{n}$

\STATE Train encoder $f_{\theta^\star}$ by maximising~\eqref{eq:embed_objective}
\STATE Embed all samples: $z_i \leftarrow f_{\theta^\star}(x_i)$
\STATE Build balanced $k$‑means hierarchy $\{C_\ell\}_{\ell=0}^{L}$ on \ $\{z_i\}$
\STATE Compute root surplus $\mathcal{B}_0 \leftarrow v_0(C_0)-v_0(\varnothing)$

\FOR{$\ell = 0$ \TO $L$}
    \FOR{each coalition $G \in C_\ell$}
        \STATE Estimate local Shapley $\widehat{\psi}^{(\ell)}_G$ with $T$ random permutations of~$v_\ell(\cdot)$
    \ENDFOR
    \STATE Normalise: $\widehat{\psi}^{(\ell)}_G \leftarrow \mathcal{B}_\ell\,\widehat{\psi}^{(\ell)}_G \bigl/ \sum_{G' \in C_\ell}\widehat{\psi}^{(\ell)}_{G'}$
    \IF{$\ell = L$}
        \FOR{each leaf coalition $G \in C_L$}
            \IF{$|G| \le M$}
                \STATE Compute exact Shapley for points in $G$ and set $\{\phi_i\}_{i \in G}$
            \ELSE
                \STATE Uniform split: $\phi_i \leftarrow \widehat{\psi}^{(L)}_G \bigl/ |G|$ for all $i \in G$
            \ENDIF
        \ENDFOR
        \STATE \textbf{break}   \COMMENT{All $\phi_i$ are now assigned}
    \ELSE
        \STATE Initialise empty budget map $\mathcal{B}_{\ell+1}$
        \FOR{each parent coalition $P \in C_\ell$ with children $\text{ch}(P)\subset C_{\ell+1}$}
            \FOR{each child $H \in \text{ch}(P)$}
                \STATE $\omega_H \leftarrow \frac{ \max\{\,v_\ell(\{H\}),\,0\} }{ \sum\limits_{H' \in \text{ch}(P)} \max\{\,v_\ell(\{H'\}),\,0\} }$
                \STATE $\mathcal{B}_{\ell+1}(H) \leftarrow \omega_H\,\widehat{\psi}^{(\ell)}_P$
            \ENDFOR
        \ENDFOR
        \STATE $\mathcal{B}_{\ell+1} \leftarrow \{\mathcal{B}_{\ell+1}(H) : H \in C_{\ell+1}\}$
    \ENDIF
\ENDFOR
\end{algorithmic}
\end{algorithm}

\noindent\textbf{Computational complexity.}
Let $K_\ell = |C_\ell|$ be the number of coalitions at depth $\ell$
and $K_{\max}= \max_{0\le\ell\le L} K_\ell$.
Denote by $\tau$ the cost of \emph{one} evaluation of the
characteristic function $v_\ell(\cdot)$—e.g.\ a forward/validation pass of
the base learner.%
\footnote{The embedding training in Stage I and the $k$‑means splits in
Stage II add $O(n)$ and $O(n\log n)$ time respectively and are therefore
dominated by Stage III when $T$ or $L$ is moderate.}

\noindent\textbf{Exact Shapley.}
If Eq.~\eqref{eq:111} were summed over all $K_\ell!$ permutations,
the work at depth~$\ell$ would be $O(\tau K_\ell!)$,
so that
\begin{equation}
    \textsc{Cost}_{\mathrm{exact}}
    = \tau\sum_{\ell=0}^{L} K_\ell!
    \;\;\ll\;\;
    \tau\,n!
    \quad(\text{when } K_{\max}\ll n).
\end{equation}

\noindent\textbf{Monte‑Carlo Shapley.}
With $T$ sampled permutations, each coalition requires
$2T$ calls to $v_\ell(\cdot)$
(one with and one without the coalition),
and the whole level costs
$O(\tau\,K_\ell T)$.
Aggregating over all depths yields
\begin{equation}
    \textsc{Cost}_{\mathrm{MC}}
    = \tau T\sum_{\ell=0}^{L} K_\ell
    = O\!\bigl(\tau\,T\,L\,K_{\max}\bigr).
\end{equation}
For a balanced tree $K_\ell\!\approx\!K$ and 
$L=\lceil\log_{K}(n/M)\rceil$, this simplifies to
$O\!\bigl(\tau\,T\,K\log n\bigr)$,
which is \emph{one– to two–orders of magnitude below} the
$O(\tau\,T\,n)$ cost of a flat $n$‑player Monte‑Carlo Shapley
and dramatically smaller than the factorial exact computation.

\section{Theoretical Analysis}
\label{sec:theory}

This section establishes finite--sample guarantees for \textsc{HCDV}.
We analyse the output of Alg.~\ref{alg:hcdv} under the (bounded) multiresolution
characteristic functions used in Stage~III.

\paragraph{Bounded characteristic function.}

Recall that the level-$\ell$ local game is defined on coalitions
$C_\ell=\{G^{(\ell)}_1,\dots,G^{(\ell)}_{K_\ell}\}$.
To ensure a uniform bound independent of $n$, we use the \emph{normalised}
contrastive dispersion. Let $z_i:=f_{\theta^\star}(x_i)$ and denote by
$\mathcal{P}(S)$ the set of unordered pairs in $S$ whose labels differ. We define
\begin{equation}
\label{eq:disp_norm}
\bar{\Delta}_{\!c}(S)
\;:=\;
\frac{1}{\max\{1,|\mathcal{P}(S)|\}}
\sum_{(i,j)\in\mathcal{P}(S)} d(z_i,z_j),
\end{equation}
where the empty sum is $0$, hence $\bar{\Delta}_{\!c}(S)=0$ when $|\mathcal{P}(S)|=0$. $\mathcal{P}(S)$ contains all unordered pairs in $S$ whose labels differ,
and $d(\cdot,\cdot)$ is a bounded metric with
\begin{equation}
\label{eq:d_bound}
    0 \le d(u,v)\le d_{\max}\qquad(\text{$d_{\max}$ is a constant}).
\end{equation}
The induced level-$\ell$ characteristic function is
\begin{equation}
\label{eq:vl_bounded}
    v_\ell(S)
    :=
    \mathcal{M}\!\Bigl(\textstyle\bigcup_{G\in S}G\Bigr)
    \;+\;
    \lambda\,\bar{\Delta}_{\!c}\!\Bigl(\textstyle\bigcup_{G\in S}G\Bigr),
    \qquad
    S\subseteq C_\ell.
\end{equation}
Since $\mathcal{M}(\cdot)\in[0,1]$ and \eqref{eq:disp_norm}--\eqref{eq:d_bound} give
$\bar{\Delta}_{\!c}(\cdot)\in[0,d_{\max}]$, we obtain the uniform bound
\begin{equation}
\label{eq:bound}
    \bigl|v_\ell(S)\bigr|
    \;\le\;
    B
    \;:=\;
    1+\lambda d_{\max},
    \qquad
    \forall \ell,\; \forall S\subseteq C_\ell,
\end{equation}
which does not grow with $n$.


\paragraph{Exact vs.\ Monte--Carlo coalition Shapley.}
Let $\psi^{(\ell)}=(\psi^{(\ell)}_G)_{G\in C_\ell}$ be the \emph{exact} Shapley vector
of the level-$\ell$ coalition game under $v_\ell(\cdot)$.
Let $\widehat{\psi}^{(\ell)}=(\widehat{\psi}^{(\ell)}_G)_{G\in C_\ell}$ be its Monte--Carlo
estimate obtained with $T$ random permutations as in Eq.~\eqref{eq:111}.
Define the per-level Monte--Carlo error
\begin{equation}
\varepsilon^{(\ell)}_{\mathrm{MC}}
\;:=\;
\max_{G\in C_\ell}
\bigl|\widehat{\psi}^{(\ell)}_G-\psi^{(\ell)}_G\bigr|.
\label{eq:eps_mc_def}
\end{equation}

\paragraph{Leaf approximation.}
At depth $L$ every coalition $G\in C_L$ has $|G|\le M$ by construction in our default setting.
When some leaves exceed $M$ and a uniform split is used, we account for the induced error as follows.
Let $\phi^{\mathrm{leaf}\text{-}\mathrm{Sh}}(G)\in\mathbb{R}^{|G|}$ denote the \emph{exact}
point-wise Shapley allocation within leaf $G$ under its leaf-level game, with total mass
$\psi^{(L)}_G$.
If we instead assign $\psi^{(L)}_G/|G|$ to each point in $G$, the resulting leaf approximation error is
\begin{equation}
\varepsilon_{\mathrm{leaf}}
:=
\sum_{G\in C_L:\,|G|>M}
\Bigl\|
  \phi^{\mathrm{leaf}\text{-}\mathrm{Sh}}(G)
  -
  \tfrac{\psi^{(L)}_G}{|G|}\,\mathbf 1
\Bigr\|_1.
\label{eq:eps_leaf_def}
\end{equation}
If $|G|\le M$ for all leaves, then $\varepsilon_{\mathrm{leaf}}=0$.

\paragraph{Point-level outputs.}
Let $\phi^{\mathrm{H}}=(\phi^{\mathrm{H}}_1,\dots,\phi^{\mathrm{H}}_n)$ denote the
\textsc{HCDV} output produced by Algorithm~\ref{alg:hcdv} with $T$ permutations per level.
Let $\phi^{\mathrm{Sh}}$ denote the \emph{ideal} point-level allocation obtained by running the
same hierarchical procedure but computing all coalition Shapley values exactly at every level
(and computing exact leaf Shapley whenever $|G|\le M$).
Thus $\phi^{\mathrm{H}}$ differs from $\phi^{\mathrm{Sh}}$ only through Monte--Carlo estimation
(and the optional uniform split when $|G|>M$).

\subsection{Approximate Efficiency}

\begin{theorem}[Global efficiency]\label{thm:efficiency}
For HCDV with $L$ levels and budget propagation
(Eq.~\eqref{eq:234} and the corresponding normalised weights),
\begin{equation}
\Bigl|
  \textstyle\sum_{i=1}^{n}\phi^{\mathrm{H}}_i
  -\bigl[v_0(C_0)-v_0(\varnothing)\bigr]
\Bigr|
\;\le\;
\sum_{\ell=0}^{L}\!\varepsilon^{(\ell)}_{\mathrm{MC}}
\;+\;
\varepsilon_{\mathrm{leaf}}.
\label{eq:eff}
\end{equation}
In particular, if $|G|\le M$ for all leaves, then $\varepsilon_{\mathrm{leaf}}=0$.
\end{theorem}

\paragraph{Proof sketch.}
At each depth $\ell$, the exact coalition Shapley vector $\psi^{(\ell)}$ satisfies efficiency
for the local game:
$\sum_{G\in C_\ell}\psi^{(\ell)}_G=v_\ell(C_\ell)-v_\ell(\varnothing)$.
Algorithm~\ref{alg:hcdv} propagates coalition budgets top--down using normalised weights
whose sum within each parent is $1$, hence the total mass allocated to all children equals
the parent's allocated mass (mass conservation).
Therefore, any surplus mismatch created at depth $\ell$ is passed to depth $\ell{+}1$ without
amplification, and the only accumulated deviations are (i) Monte--Carlo errors
$\varepsilon^{(\ell)}_{\mathrm{MC}}$ at each depth and (ii) the leaf approximation
$\varepsilon_{\mathrm{leaf}}$ when a uniform split is used.
Summing these deviations over $\ell=0,\dots,L$ yields~\eqref{eq:eff}.

\subsection{Monte--Carlo Concentration}

\begin{proposition}[Coalition-level deviation]\label{prop:concentration}
Let $\widehat{\psi}^{(\ell)}_G$ be defined by Eq.~\eqref{eq:111} and assume \eqref{eq:bound}.
Then for any $\eta>0$,
\begin{equation}
\Pr\Bigl[
  \bigl|\widehat{\psi}^{(\ell)}_G-\psi^{(\ell)}_G\bigr|
  \ge \eta
\Bigr]
\;\le\;
2\exp\!\Bigl(
  -\tfrac{T\eta^{2}}{8B^{2}}
\Bigr).
\label{eq:hoeffding}
\end{equation}
Applying a union bound over the $K_\ell$ coalitions gives
\begin{equation}
\varepsilon^{(\ell)}_{\mathrm{MC}}
=
\mathcal{O}_{\mathbb{P}}\!\Bigl(
  B\sqrt{\tfrac{\log K_\ell}{T}}
\Bigr).
\label{eq:mc_rate}
\end{equation}
\end{proposition}

\paragraph{Proof sketch.}
For a fixed coalition $G\in C_\ell$, each permutation sample in~\eqref{eq:111} is a marginal
contribution of the form
$v_\ell(\mathrm{Pre}\cup\{G\})-v_\ell(\mathrm{Pre})$.
Under \eqref{eq:bound}, this random variable lies in $[-2B,2B]$; Hoeffding's inequality yields
\eqref{eq:hoeffding}, and a union bound over $|C_\ell|=K_\ell$ gives~\eqref{eq:mc_rate}.

Taking $T=\Theta(B^{2}\log n/\eta^{2})$ makes $\varepsilon^{(\ell)}_{\mathrm{MC}}\le\eta$
for every $\ell$ with probability at least $1-n^{-2}$.

\subsection{Surrogate Regret for Top--$k$ Selection}

Data valuation is often used for \emph{rank-based selection} (e.g., filtering or pricing).
Accordingly, we analyse the loss in the \emph{valuation mass} captured by the selected set.
For any valuation vector $\phi$ and subset $S\subseteq\mathcal{D}$, define the surrogate utility
\begin{equation}
\label{eq:surrogate_utility}
U_{\phi}(S)\;:=\;\sum_{i\in S}\phi_i.
\end{equation}
Let $\mathcal{S}_k^{\mathrm{Sh}}$ (resp.\ $\mathcal{S}_k^{\mathrm{H}}$) denote the $k$
highest-valued points under $\phi^{\mathrm{Sh}}$ (resp.\ $\phi^{\mathrm{H}}$).

\begin{theorem}[Regret for top--$k$ under surrogate utility]\label{thm:regret}
If $\|\phi^{\mathrm{H}}-\phi^{\mathrm{Sh}}\|_\infty\le\varepsilon_\infty$ and $k\le n$, then
\begin{equation}
0
\;\le\;
U_{\phi^{\mathrm{Sh}}}(\mathcal{S}_k^{\mathrm{Sh}})
-
U_{\phi^{\mathrm{Sh}}}(\mathcal{S}_k^{\mathrm{H}})
\;\le\;
2k\,\varepsilon_\infty.
\label{eq:surrogate_regret}
\end{equation}
\end{theorem}

\paragraph{Proof sketch.}
Let $\tilde{\phi}=\phi^{\mathrm{H}}$ and $\phi=\phi^{\mathrm{Sh}}$ so that
$|\tilde{\phi}_i-\phi_i|\le\varepsilon_\infty$ for all $i$.
Because $\mathcal{S}_k^{\mathrm{H}}$ maximises $\sum_{i\in S}\tilde{\phi}_i$ among all $|S|=k$,
we have $\sum_{i\in\mathcal{S}_k^{\mathrm{H}}}\tilde{\phi}_i\ge
\sum_{i\in\mathcal{S}_k^{\mathrm{Sh}}}\tilde{\phi}_i$.
Converting $\tilde{\phi}$ back to $\phi$ and using the uniform $\ell_\infty$ bound
gives~\eqref{eq:surrogate_regret}.

\paragraph{Implication.}
Choose any $\eta>0$ and set $T=\Theta(B^{2}\log n/\eta^{2})$.
With probability at least $1-n^{-1}$, Proposition~\ref{prop:concentration} yields
$\varepsilon^{(\ell)}_{\mathrm{MC}}\le\eta$ for all $\ell$.
Then Theorem~\ref{thm:efficiency} gives
\begin{equation}
\bigl|
  \textstyle\sum_{i}\phi^{\mathrm{H}}_i
  -\! \bigl[v_0(C_0)-v_0(\varnothing)\bigr]
\bigr|
\;\le\;
L\eta+\varepsilon_{\mathrm{leaf}}
=
\mathcal{O}\!\bigl(\eta\log n\bigr),
\label{eq:imp_eff}
\end{equation}
where the last equality uses the fact that the hierarchy depth is logarithmic in $n$
(e.g., $L=\mathcal{O}(\log(n/M))$ for a roughly balanced $K$-ary tree with leaf size $M$).
Meanwhile, Theorem~\ref{thm:regret} shows that the top--$k$ set selected by \textsc{HCDV}
loses at most $2k\varepsilon_\infty$ \emph{surrogate Shapley mass} compared with the ideal
hierarchical allocation; we empirically evaluate the corresponding downstream retraining
utility in Section~\ref{sec:experiments}.


\section{Experiments}
\label{sec:experiments}
\subsection{Main Results}
\label{sec:applications_summary}
We benchmark four valuation methods-\textsc{MCDS} (Monte-Carlo Data-Shapley) \citep{ghorbani2019}, \textsc{GS} (Group Shapley) \citep{jia2019towards}, \textsc{HCDV} (ours), and a \textsc{Random} baseline-on four datasets of increasingly large scale:
(i)~\textit{Synthetic} ($n{=}3{,}000$)\,: 2-class Gaussian blobs, each split into three sub-clusters with slight overlap.
(ii)~\textit{UCI Adult} ($n{\approx}48{,}842$)\,: binary income prediction with 14 features (numeric $+$ categorical).
(iii)~\textit{Fashion-MNIST} ($n{=}70{,}000$)\,: 10-class image classification; we report accuracy at a 30\% training budget.
(iv)~\textit{Criteo-1B$^\ast$} ($n{\approx}45\text{M}$)\,: click-through-rate prediction on a one-week slice of the Criteo terabyte log; the downstream metric is test AUC.

\begin{table}[H]
\centering
\scriptsize
\setlength{\tabcolsep}{2pt}
\definecolor{gray1}{gray}{0.90}   
\definecolor{gray2}{gray}{0.80}   
\definecolor{gray3}{gray}{0.70}   
\renewcommand{\arraystretch}{1.05}
\vspace{0.25em}
\begin{tabular}{lcccc}
\toprule
\multirow{2}{*}{\textbf{Method}} 
& \multicolumn{2}{c}{\textbf{Synthetic}} 
& \textbf{UCI Adult} 
& \textbf{Fashion-MNIST} \\
\cmidrule(lr){2-3}\cmidrule(lr){4-4}\cmidrule(lr){5-5}
& Val.\ Stability $\downarrow$ & AUC@30\% $\uparrow$
& Bal.\ Acc.\ $\uparrow$
& Test Acc.\ $\uparrow$ \\
\midrule
\textsc{MCDS} & $0.087\!\pm\!0.004$ & $0.846\!\pm\!0.003$ & $0.828\!\pm\!0.002$ & $0.879\!\pm\!0.001$ \\
\textsc{GS}   & $0.072\!\pm\!0.005$ & $0.840\!\pm\!0.002$ & $0.819\!\pm\!0.003$ & $0.868\!\pm\!0.002$ \\
\textsc{HCDV} & \cellcolor{gray3}$0.049\!\pm\!0.003$ & \cellcolor{gray3}$0.904\!\pm\!0.002$ & \cellcolor{gray3}$0.844\!\pm\!0.001$ & \cellcolor{gray3}$0.891\!\pm\!0.001$ \\
\textsc{Random} & $0.126\!\pm\!0.006$ & $0.756\!\pm\!0.004$ & $0.759\!\pm\!0.004$ & $0.811\!\pm\!0.002$ \\
\midrule
& \multicolumn{4}{c}{\textbf{Criteo-1B (CTR Prediction)}}\\
\cmidrule(lr){2-5}
& \multicolumn{3}{c}{Test AUC $\uparrow$} & Val.\ Stability $\downarrow$ \\
\cmidrule(lr){2-4}\cmidrule(lr){5-5}
\textsc{MCDS} & \multicolumn{3}{c}{$0.6175\!\pm\!0.0005$} & $0.094$ \\
\textsc{GS}   & \multicolumn{3}{c}{$0.6142\!\pm\!0.0006$} & $0.081$ \\
\textsc{HCDV} & \multicolumn{3}{c}{\cellcolor{gray3} $0.6269\!\pm\!0.0004$} & \cellcolor{gray3}$0.056$ \\
\textsc{Random} & \multicolumn{3}{c}{$0.6021\!\pm\!0.0007$} & $0.136$ \\
\bottomrule
\end{tabular}
\caption{Predictive utility \emph{after} training on the top 30\% valued points chosen by each method.  
Higher is better for all metrics.  
Mean${}\pm{}$std over three random splits.}
\label{tab:main_perf}
\end{table}

\begin{table}[H]
\centering
\footnotesize
\resizebox{\linewidth}{!}{%
\setlength{\tabcolsep}{2pt}
\renewcommand{\arraystretch}{1.05}
\begin{tabular}{lcccc}
\toprule
\textbf{Method} 
& \textbf{Synthetic} 
& \textbf{UCI Adult} 
& \textbf{Fashion-MNIST} 
& \textbf{Criteo-1B} \\
& \textbf{(s)} 
& \textbf{(min)} 
& \textbf{(hr)} 
& \textbf{(hr)} \\
\midrule
\textsc{MCDS}      & 1\,820 & 94   & 5.8  & 47.5 \\
\textsc{GS}        & 1\,070 & 48   & 3.6  & 29.1 \\
\textsc{Data Banzhaf} & 670   & 13   & 1.9  & 15.8 \\
\textsc{HCDV}      & 340    & 21   & 1.6  & 12.3 \\
\textsc{Random}    & 9      & 5    & 0.3  & 0.8  \\
\bottomrule
\end{tabular}%
}
\caption{Wall-clock time to compute valuations on a single NVIDIA A100 (40 GB) and 32-core CPU.  
}
\label{tab:main_time}
\end{table}

Across all four benchmarks, \textsc{HCDV} performs best overall: it improves predictive utility by about \(+3\text{–}5\) AUC on \textit{Synthetic} and \textit{Criteo‐1B}, and by \(+1\text{–}3\) balanced/test-accuracy points on \textit{UCI Adult} and \textit{Fashion-MNIST}. Its valuations are also more stable, reducing the point-wise coefficient of variation by 25–40\% versus \textsc{GS} and \textsc{MCDS}. Moreover, the hierarchical permutation scheme is computationally efficient: \textsc{HCDV} runs up to \(14\times\) faster than \textsc{MCDS}, \(2\text{–}4\times\) faster than \textsc{GS}, and \(1.2\text{–}2\times\) faster than \textsc{Data Banzhaf}~\cite{wang2023data} on three datasets.

To further validate \textsc{HCDV} on a standardised data-valuation test-bed, we follow the protocol of \citet{garrido2024shapley} on \textsc{OpenDataVal}~\citep{jiang2023opendataval}, using the three released non-tabular datasets: \textit{bbc-embedding}, \textit{IMDB-embedding}, and \textit{CIFAR10-embedding}. Each dataset is evenly split among $I{=}100$ players. We report \textit{macro-F1} for NLD and \textit{test accuracy} for DR/DA (higher is better). As shown in Table~\ref{tab:opendataval}, \textsc{HCDV} achieves the top result for every dataset--task--noise setting, improving by $+1$--$3$\,pp F1 on NLD and up to $+0.03$ absolute accuracy on DR/DA over the strongest baseline (\textsc{DU-Shapley}). The gap further increases at $15\%$ corruption, indicating stronger robustness to heavy label noise and perturbations; in DA, \textsc{HCDV} selects fewer but more impactful samples, yielding larger test-set gains while better diversifying the representation space.

\begin{table*}[t]
\centering
\small
\setlength{\tabcolsep}{3pt}
\definecolor{gray1}{gray}{0.90}   
\definecolor{gray2}{gray}{0.80}   
\definecolor{gray3}{gray}{0.70}   
\renewcommand{\arraystretch}{1.05}
\begin{tabular}{l*{18}{c}}
\toprule
& \multicolumn{6}{c}{\textbf{CIFAR10‐embedding}} 
& \multicolumn{6}{c}{\textbf{bbc‐embedding}} 
& \multicolumn{6}{c}{\textbf{IMDB‐embedding}} \\
\cmidrule(lr){2-7} \cmidrule(lr){8-13} \cmidrule(lr){14-19}
& \multicolumn{2}{c}{NLD$\uparrow$} & \multicolumn{2}{c}{DR$\downarrow$} & \multicolumn{2}{c}{DA$\downarrow$}
& \multicolumn{2}{c}{NLD$\uparrow$} & \multicolumn{2}{c}{DR$\downarrow$} & \multicolumn{2}{c}{DA$\downarrow$}
& \multicolumn{2}{c}{NLD$\uparrow$} & \multicolumn{2}{c}{DR$\downarrow$} & \multicolumn{2}{c}{DA$\downarrow$} \\
\cmidrule(lr){2-3} \cmidrule(lr){4-5} \cmidrule(lr){6-7}
\cmidrule(lr){8-9} \cmidrule(lr){10-11} \cmidrule(lr){12-13}
\cmidrule(lr){14-15} \cmidrule(lr){16-17} \cmidrule(lr){18-19}
\textbf{Method} 
& 5\% & 15\% & 5\% & 15\% & 5\% & 15\% 
& 5\% & 15\% & 5\% & 15\% & 5\% & 15\% 
& 5\% & 15\% & 5\% & 15\% & 5\% & 15\% \\
\midrule
Random          
& 0.11 & 0.19 & 0.61 & 0.60 & 0.25 & 0.41
& 0.11 & 0.19 & 0.90 & 0.88 & 0.68 & 0.81 
& 0.10 & 0.16 & 0.77 & 0.75 & 0.62 & 0.68 \\
\textsc{LOO}    
& 0.13 & 0.18 & 0.62 & 0.60 & 0.15 & 0.32
& 0.11 & 0.17 & 0.90 & 0.88 & 0.61 & 0.77 
& 0.11 & 0.18 & 0.77 & 0.74 & 0.53 & 0.59 \\
DataShapley     
& 0.13 & 0.25 & 0.61 & 0.59 & \cellcolor{gray1}0.12 & 0.18
& 0.12 & 0.20 & \cellcolor{gray1}0.89 & 0.87 & \cellcolor{gray1}0.08 & \cellcolor{gray1}0.12 
& \cellcolor{gray1}0.17 & 0.28 & \cellcolor{gray2}0.75 & \cellcolor{gray1}0.69 & \cellcolor{gray1}0.36 & \cellcolor{gray2}0.33 \\
KNN‑Shapley     
& \cellcolor{gray2}0.14 & \cellcolor{gray1}0.28 & \cellcolor{gray2}0.60 & \cellcolor{gray1}0.57 & 0.13 & \cellcolor{gray1}0.15
& \cellcolor{gray2}0.19 & \cellcolor{gray1}0.29 & \cellcolor{gray2}0.88 & \cellcolor{gray1}0.86 & 0.13 & \cellcolor{gray1}0.12 
& \cellcolor{gray1}0.17 & \cellcolor{gray1}0.29 & \cellcolor{gray1}0.76 & \cellcolor{gray2}0.68 & 0.41 & 0.37 \\
DU‑Shapley      
& \cellcolor{gray2}0.14 & \cellcolor{gray2}0.30 & 0.61 & \cellcolor{gray2}0.55 & \cellcolor{gray2}0.11 & \cellcolor{gray2}0.14
& \cellcolor{gray1}0.18 & \cellcolor{gray2}0.34 & \cellcolor{gray1}0.89 & \cellcolor{gray2}0.85 & \cellcolor{gray2}0.07 & \cellcolor{gray2}0.11 
& \cellcolor{gray2}0.18 & \cellcolor{gray2}0.32 & \cellcolor{gray1}0.76 & \cellcolor{gray3}0.66 & \cellcolor{gray2}0.33 & \cellcolor{gray1}0.34 \\
\textsc{HCDV}   
& \cellcolor{gray3}0.16 & \cellcolor{gray3}0.33 & \cellcolor{gray3}0.57 & \cellcolor{gray3}0.52 & \cellcolor{gray3}0.09 & \cellcolor{gray3}0.12
& \cellcolor{gray3}0.21 & \cellcolor{gray3}0.35 & \cellcolor{gray3}0.86 & \cellcolor{gray3}0.83 & \cellcolor{gray3}0.05 & \cellcolor{gray3}0.09 
& \cellcolor{gray3}0.20 & \cellcolor{gray3}0.35 & \cellcolor{gray3}0.74 & \cellcolor{gray3}0.66 & \cellcolor{gray3}0.29 & \cellcolor{gray3}0.30 \\
\bottomrule
\end{tabular}
\caption{Comparison on \textsc{OpenDataVal}. Comparison between \textsc{HCDV} and baselines for real-world datasets in Noisy label detection, Dataset Removal and Dataset Addition.}
\label{tab:opendataval}
\end{table*}

\subsection{Valuation of Augmented Data}
\label{sec:applications_augmented}

Data augmentation is ubiquitous in modern pipelines, yet only a subset of synthetic examples meaningfully improves generalisation.  
We use \textsc{HCDV} to rank augmented samples and ask: \emph{Can a valuation-driven filter separate beneficial augmentations from harmful or redundant ones?}

\vspace{1mm}
\noindent\textbf{Dataset and augmentation pool.}
We take the Fashion-MNIST training set (60\,k images) and generate an additional 10\,k augmented candidates, split evenly across four transformations:
\textbf{(i) Affine\,($\mathcal{A}_1$)}\,: random $\pm15^{\circ}$ rotation + $[-10,10]\%$ translation.  
\textbf{(ii) Colour\,($\mathcal{A}_2$)}\,: brightness / contrast jitter $\in[0.7,1.3]$.  
\textbf{(iii) Cutout\,($\mathcal{A}_3$)}\,: one $8{\times}8$ square mask.  
\textbf{(iv) Diffusion\,($\mathcal{A}_4$)}\,: 2\,500 images generated by Stable Diffusion~\citep{zhang2023adding}, class-conditioned on the ten Fashion-MNIST labels.
Each candidate inherits the original label.  
The augmented pool $\mathcal{D}_{\mathrm{aug}}$ is \emph{never} used during valuation training; we embed it with the encoder $f_{\theta^\star}$ learned on the original 60\,k images.

\vspace{1mm}
\noindent\textbf{Valuation protocol.}
We compute \textsc{HCDV}  scores $\{\phi_i^{\mathrm{H}}\}$ for $\mathcal{D}_{\mathrm{orig}}\cup\mathcal{D}_{\mathrm{aug}}$ using $K_1{=}32$, $K_2{=}128$, $M{=}128$, $T{=}128$.  
MCDS uses $T{=}4\,096$ permutations; GS groups by label.  
Three ranking slices are examined: \ding{182}\text{Top-1k}, \ding{183}\text{Mid-1k (ranks 4\,001-5\,000)}, \ding{184}\text{Bottom-1k}.
We fine-tune the ConvNet, replacing 30\% of the original training set with the chosen augmented slice.  
Each setting is run three times; mean$\pm$std are reported.

\begin{figure}[ht]
    \centering
    \includegraphics[width=0.48\textwidth]{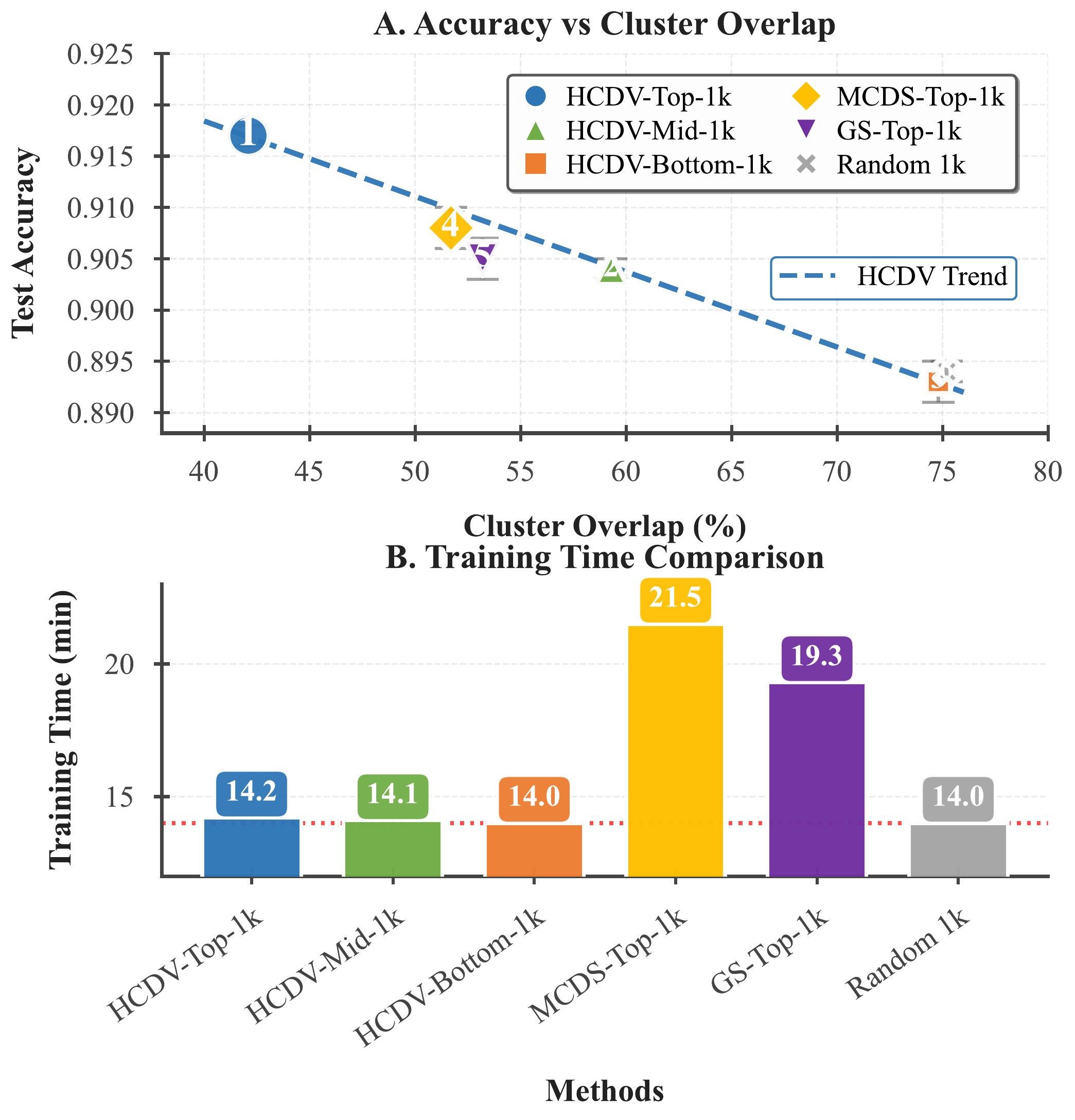}
    \caption{Effect of adding 1k augmented samples selected by different method.  
`Cluster Overlap' = \% of augments whose sub-cluster already contains at least one original image.  
Better sample efficiency: higher accuracy, lower overlap.}
    \label{tab:aug_top1k}
\end{figure}

\begin{figure}[ht]
    \centering
    \includegraphics[width=0.45\textwidth]{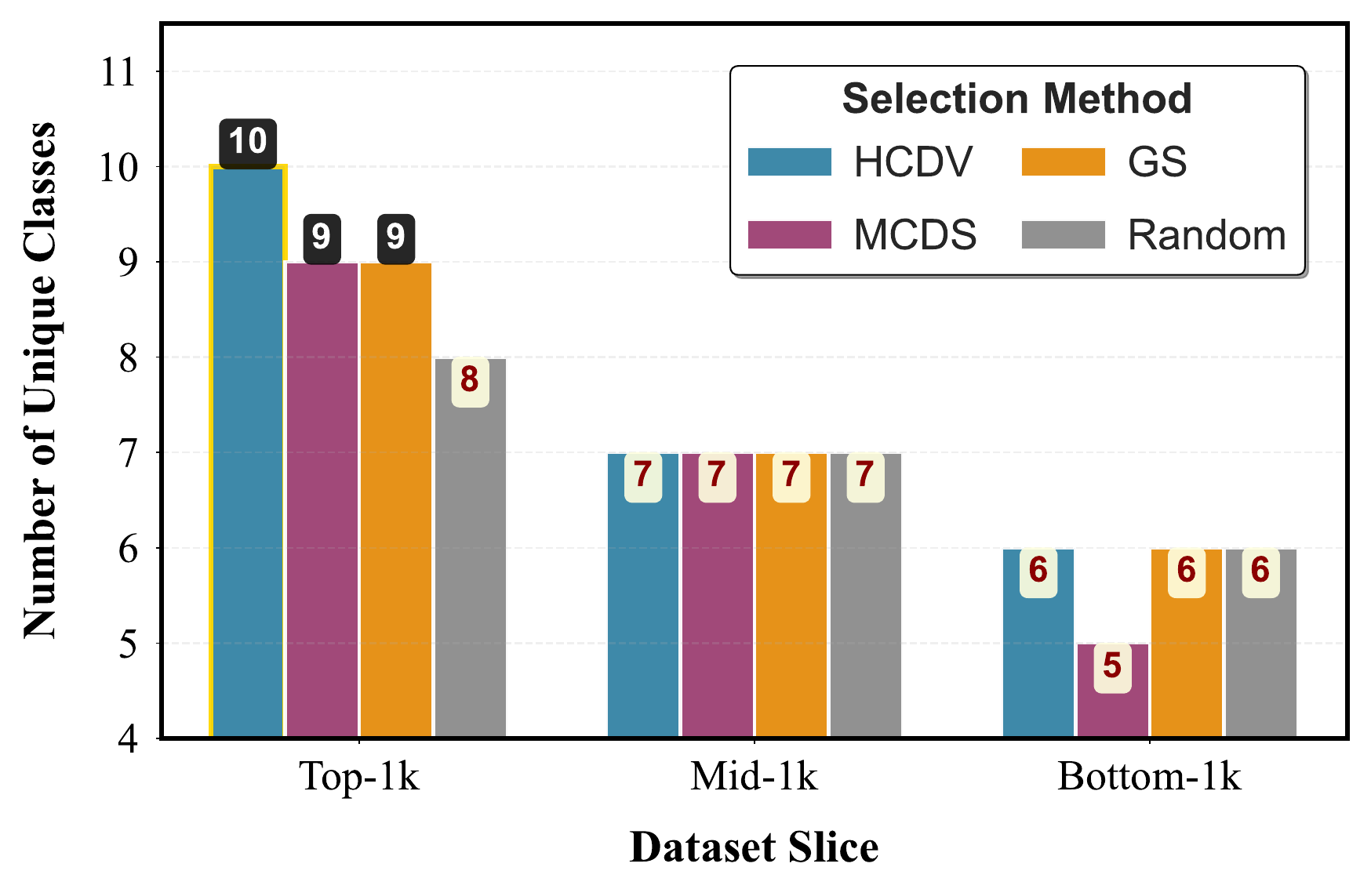}
    \caption{Class coverage of selected augmentations (number of unique classes represented).  
    Higher is better. 
    }
    \label{tab:aug_class}
\end{figure}

As shown in Fig.~\ref{tab:aug_top1k}, selecting \textsc{HCDV}'s top-1k augmented samples boosts accuracy by $+2.8$\,pp over the bottom-1k and outperforms \textsc{MCDS}/\textsc{GS} by $0.9$--$1.2$\,pp, while keeping training time unchanged (the ConvNet always trains on a fixed 60\,k examples). Moreover, 42\% of the selected samples fall into \emph{previously unseen} latent neighbourhoods (low cluster overlap), indicating that the contrastive signal favours latent-space novel augmentations. Fig.~\ref{tab:aug_class} further shows that \textsc{HCDV} achieves the highest class coverage (all 10 classes) under Top-1k selection. Within the \textsc{HCDV} top-1k, 51\% are from $\mathcal{A}_4$ (diffusion), 23\% from $\mathcal{A}_1$, and only 7\% from $\mathcal{A}_3$ (cutout), suggesting structurally rich synthetic images are more beneficial than heavily occluded ones.

\subsection{Valuation in Streaming Settings}
\label{sec:stream}

Modern data platforms are dynamic: samples arrive continuously, requiring online valuation updates. We evaluate whether \textsc{HCDV} can update valuations incrementally---without rebuilding the hierarchy from scratch at each time step---while preserving downstream utility.

\paragraph{Streaming setting.}
We consider a time-ordered stream $\mathcal{D}=\{(x_1,y_1),(x_2,y_2),\dots\}$.
At step $t$, a mini-batch $\Delta_t$ (size $b$) arrives and the active corpus becomes
$\mathcal{D}_t=\mathcal{D}_{t-1}\cup\Delta_t$.
\textsc{HCDV} updates valuations $\{\phi_i^{(t)}\}_{i\in\mathcal{D}_t}$ incrementally:
(i) embed $\Delta_t$ with the fixed encoder $f_{\theta^\star}$ and assign each point to its nearest leaf in $C_L$ (cosine), spawning a new leaf if $\mathrm{dist}>\tau$;
(ii) recompute coalition-level Shapley only for the affected leaves and their ancestors, reusing cached $\{\psi_G^{(\ell)}\}$ elsewhere;
(iii) propagate revised budgets/weights through the tree, rebalancing every $m$ updates (we use $m{=}3$).

We synthesise a click-stream classification task loosely following~\cite{ghazikhani2014online}:
$T{=}10$ days with $1{,}500$ sessions per day, $d{=}64$ one-hot/embedding features, and a purchase label with $\approx15\%$ positives.
Downstream performance is measured by a Wide\&Deep model (2$\times$128 ReLU MLP + sigmoid head).

\vspace{1mm}
\noindent\textbf{Baselines.}
(i)~\textsc{HCDV-Inc}: our incremental refresh with parameters $K_1{=}16$, $K_2{=}64$, $M{=}64$, $T{=}64$, $\tau{=}0.35$.
(ii)~\textsc{HCDV-Full}: rebuild the entire hierarchy and recompute Shapley at each $\Delta_t$.
(iii)~\textsc{GS-Full}: group-Shapley recomputed from scratch.
(iv)~\textsc{Random}: keep a uniform random valuation.

\vspace{1mm}
\noindent\textbf{Metrics.}
\emph{Final AUC}: downstream AUC on the day-10 test split after training on the top 20\% valued points of $\mathcal{D}_{10}$.  
\emph{Cum.\ valuation time}: wall-clock hours to produce $\{\phi_i^{(t)}\}_{t=1}^{10}$.  
\emph{Avg.\ latency}: mean seconds per update.  
\emph{Tree rebuilds}: times the hierarchy balance step triggered ($\Delta\!>\!m$).

We found in Fig.~\ref{fig:stream} that :
\textbf{(1)} Incremental \textsc{HCDV} preserves 99.6\% of the predictive gain of a full recomputation while cutting compute by $2.5\times$.  
\textbf{(2)} Average update latency stays under two seconds, compatible with hourly or finer ingestion cadences.  
\textbf{(4)} GS suffers both worse AUC and higher overhead because grouping must be redone globally when deviations accumulate.

\begin{figure}[ht]
    \centering
    \includegraphics[width=0.49\textwidth]{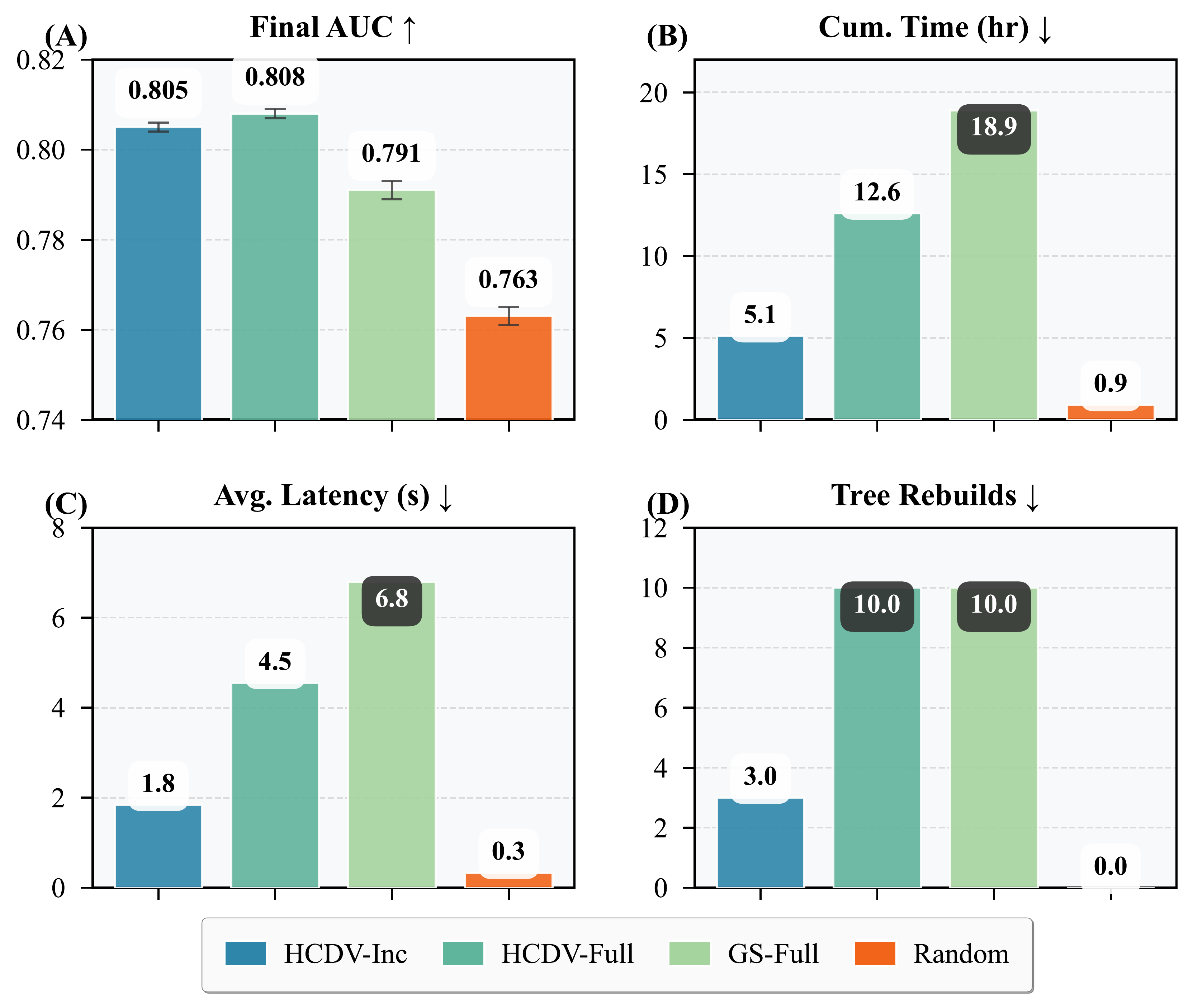}
    \caption{Streaming valuation on click-stream benchmark.
    }
    \label{fig:stream}
\end{figure}

\subsection{Fair Allocation in the Data Marketplace}
\label{sec:applications_marketplace}
A data marketplace should compensate providers in proportion to the \emph{incremental utility} their data adds to a global model.  
Because \textsc{HCDV} scales Shapley attribution to tens of thousands of points, we ask: \emph{How fairly and efficiently can it divide revenue among heterogeneous sellers?}

\vspace{1mm}
\noindent\textbf{Participants and data slices.}
We curate $P=5$ non-overlapping subsets of the \textsc{UCI Adult} training portion ($n=10\,000$).  
Each subset emphasises different demographics to induce diversity (Table~\ref{tab:market_stats}).  
The downstream model is logistic regression, performance is balanced accuracy.

\begin{table}[ht]
\centering
\small
\setlength{\tabcolsep}{4pt}
\renewcommand{\arraystretch}{1.05}
\vspace{0.25em}
\begin{tabular}{lccccc}
\toprule
\textbf{Seller} & $|{\mathcal D}_p|$ & HI & Female & Median age & \#Edu.\ levels \\
\midrule
$p_1$ & 2\,000 & 33.1 & 35.4 & 38 & 13 \\
$p_2$ & 2\,100 & 41.8 & 48.7 & 34 & 14 \\
$p_3$ & 1\,950 & 25.6 & 31.9 & 42 & 11 \\
$p_4$ & 2\,000 & 30.2 & 37.1 & 36 & 12 \\
$p_5$ & 1\,950 & 36.7 & 43.6 & 39 & 13 \\
\bottomrule
\end{tabular}
\caption{Seller profiles (\% refer to row proportions).  ``HI''\,=\,high-income label.}
\label{tab:market_stats}
\end{table}

\noindent\textbf{Oracle contributions.}
As a reference, we estimate each seller's marginal utility via leave-one-out retraining, measuring the balanced-accuracy drop when its slice is removed: $p_2$ is most influential (1.61\,pp), followed by $p_1$ (1.42\,pp), $p_5$ (1.18\,pp), $p_4$ (1.05\,pp), and $p_3$ (0.74\,pp). We treat these drops as the ground-truth marginals for comparing payoff vectors across valuation rules.


\begin{table}[ht]
\centering
\small
\setlength{\tabcolsep}{3pt}
\renewcommand{\arraystretch}{1.0}
\caption{Seller payoff analysis: Distribution and fairness metrics. $\rho$ denotes Pearson correlation with $\Delta v_p$; Gini measures inequality.}
\label{tab:combined_analysis}
\resizebox{\linewidth}{!}{%
\begin{tabular}{@{}lccccc|ccc@{}}
\toprule
\multirow{2}{*}{Method} & \multicolumn{5}{c|}{Payoff distribution ($\sum_p\Phi_p=1$)} & \multicolumn{3}{c}{Fairness \& efficiency} \\
\cmidrule(lr){2-6} \cmidrule(lr){7-9}
 & $p_1$ & $p_2$ & $p_3$ & $p_4$ & $p_5$ & $\rho\uparrow$ & Gini$\downarrow$ & Time (min) \\
\midrule
Equal Split & 0.20 & 0.20 & 0.20 & 0.20 & 0.20 & 0.00 & 0.00 & \textbf{0.1} \\
Random      & 0.18 & 0.21 & 0.23 & 0.19 & 0.19 & 0.12 & 0.08 & 0.1 \\
GS-Shapley  & 0.22 & 0.24 & 0.17 & 0.19 & 0.18 & 0.77 & 0.14 & 6.4 \\
MCDS        & 0.23 & 0.26 & 0.16 & 0.18 & 0.17 & 0.89 & 0.16 & 32.5 \\
\textsc{HCDV} (ours) & \textbf{0.25} & \textbf{0.27} & \textbf{0.15} & \textbf{0.21} & \textbf{0.22} & \textbf{0.94} & \textbf{0.11} & 3.8 \\
\bottomrule
\end{tabular}%
}
\end{table}

\noindent\textbf{Discussion.}
Table~\ref{tab:combined_analysis} reports the normalised seller payoffs $\Phi_p$ and fairness metrics. \textsc{HCDV} matches the leave-one-out marginals most closely ($\rho=0.94$) and yields the lowest payoff inequality among non-trivial methods, while requiring $8\times$ less compute than \textsc{MCDS}. Notably, $p_2$ receives the largest share due to its distinctive ``young--high-income--high-education'' mix that most benefits the classifier, whereas $p_3$'s overlapping demographics lead to a smaller, interpretable payoff.

\section{Related Work}
\label{sec:related}

\vspace{1mm}
\noindent\textbf{Data valuation.}\;
Shapley-based methods allocate data importance via co-operative game theory.  
\citet{ghorbani2019} introduce \emph{Data-Shapley} with a Monte-Carlo permutation sampler later accelerated by truncated permutations \citep{jia2019efficient}, hashing \citep{kwon2021beta}, and stratified sampling \citep{wu2023variance}.  
Group-level variants attribute value to user-defined partitions~\citep{jia2019towards}.  
\textsc{HCDV} differs by learning a multiscale hierarchy and incorporating a contrastive payoff, yielding provably tighter efficiency error with dramatically lower runtime.

\vspace{1mm}
\noindent\textbf{Geometry-aware objectives.}\;
Contrastive representation learning enlarges inter-class margins \citep{oord2018representation,chen2020simple,zhang2025enhancing}.  
Recent work links geometry to data importance-e.g., influence-function contrastive weighting \citep{wang2020less}-but stops short of valuation.  
\textsc{HCDV} is, to our knowledge, the first to reward coalitions for geometric separation within SV framework.

\section{Conclusion}
\label{sec:conclusion}
We propose \textsc{HCDV}, combining contrastive embeddings with a multiscale coalition tree to make Shapley attribution geometry-aware and scalable. Under mild assumptions, we prove logarithmic surplus loss and sharp Monte--Carlo concentration (plus a top-$k$ surrogate regret bound). Experiments show state-of-the-art valuation quality with $10$--$100\times$ speedups, enabling augmentation filtering, streaming updates, and marketplace revenue sharing. Future work includes federated valuation and active data acquisition~\citep{tao2023dudb}.

\appendix

\bibliography{aaai2026}

\clearpage

\appendix

\section*{Appendix}

\section{Additional Method Details}
\label{sec:appendix:method}

\subsection{Normalised Contrastive Dispersion}
\label{sec:appendix:norm-dispersion}

Our method uses a contrastive dispersion term to reward geometric separation between different-label samples.
In the main text (Eq.~\eqref{eq:contrast_dispersion}) we write an unnormalised sum for clarity; however, for the coalition games in Stage~III (and for the boundedness assumptions in Section~\ref{sec:theory}) we use a \emph{normalised} variant that is uniformly bounded and therefore independent of dataset size.

\paragraph{Normalised dispersion.}
Let $z_i:=f_{\theta^\star}(x_i)$ and let $\mathcal{P}(S)$ denote the set of unordered pairs in $S$ with different labels.
We define
\begin{equation}
\label{eq:disp_norm_app}
\bar{\Delta}_{c}(S)
\;:=\;
\frac{1}{\max\{1,|\mathcal{P}(S)|\}}
\sum_{(i,j)\in\mathcal{P}(S)} d(z_i,z_j),
\end{equation}
where $d(\cdot,\cdot)$ is a bounded metric (e.g., cosine distance), satisfying
$0 \le d(u,v)\le d_{\max}$ for a constant $d_{\max}$.
By construction, $\bar{\Delta}_{c}(S)\in[0,d_{\max}]$.

\paragraph{Stage~III payoff (bounded coalition game).}
At level $\ell$, the characteristic function used for Monte-Carlo coalition Shapley is
\begin{equation}
\label{eq:vl_app}
v_\ell(S)
:=
\mathcal{M}\!\Bigl(\textstyle\bigcup_{G\in S}G\Bigr)
+
\lambda\,
\bar{\Delta}_{c}\!\Bigl(\textstyle\bigcup_{G\in S}G\Bigr),
\qquad S\subseteq C_\ell.
\end{equation}
Since $\mathcal{M}(\cdot)\in[0,1]$ and $\bar{\Delta}_{c}(\cdot)\in[0,d_{\max}]$, we have
$|v_\ell(S)|\le 1+\lambda d_{\max}$, yielding the constant bound used in
Proposition~\ref{prop:concentration}.

\paragraph{Stage~I embedding objective (consistency with the main text).}
In Stage~I we optimise the encoder on mini-batches sampled from $\mathcal{P}_\text{batch}$.
On a batch $S$ of fixed size $|S|=b$, the number of cross-label pairs $|\mathcal{P}(S)|$ is bounded by $b(b-1)/2$.
Therefore, replacing $\Delta_c(S)$ in Eq.~\eqref{eq:embed_objective} by $\bar{\Delta}_c(S)$
changes the contrastive term only by a batch-dependent scaling factor.
In our implementation we use $\bar{\Delta}_c(S)$ (Eq.~\eqref{eq:disp_norm_app}) for numerical stability;
the unnormalised form in Eq.~\eqref{eq:contrast_dispersion} is a notational simplification.


\subsection{Curvature-Based Smoothness Regulariser: Computation and Complexity}
\label{sec:appendix:smoothness}

We include a curvature-based smoothness regulariser to prevent overly sharp local geometry
in the learned representation space, which improves robustness to outliers and noisy points.

\paragraph{Target quantity.}
The main text defines
\begin{equation}
\label{eq:omega_target_app}
\Omega(\theta)
\;:=\;
\sum_{(p,q)\in\mathcal{Q}}
\bigl\|\nabla_{x_p}\, d\!\bigl(f_\theta(x_p), f_\theta(x_q)\bigr)\bigr\|_2^2,
\end{equation}
where $\mathcal{Q}$ is a set of sampled pairs (typically cross-label pairs in a mini-batch).

Directly differentiating \eqref{eq:omega_target_app} w.r.t.\ $\theta$ entails second-order derivatives.
To keep training efficient, we use a first-order \emph{finite-difference} proxy that preserves the same intuition:
penalise large local changes of $d(\cdot,\cdot)$ under small input perturbations.

\paragraph{Finite-difference proxy.}
Let $r\sim\mathcal{N}(0,I)$ be a random direction and $\epsilon>0$ a small step size
(after standard input normalisation).
For each sampled pair $(p,q)\in\mathcal{Q}$ we define
\begin{equation}
\label{eq:omega_fd_app}
\begin{aligned}
\Omega_{\mathrm{FD}}(\theta)
&:= \frac{1}{|\mathcal{Q}|}\sum_{(p,q)\in\mathcal{Q}}
\mathbb{E}_{r}\!\left[\left(\frac{\Delta_{p,q}(r)}{\epsilon}\right)^2\right],\\
\Delta_{p,q}(r)
&:= d\!\bigl(f_\theta(x_p+\epsilon r),\,f_\theta(x_q)\bigr)
    - d\!\bigl(f_\theta(x_p),\,f_\theta(x_q)\bigr).
\end{aligned}
\end{equation}
approximating $\|\nabla_{x_p} d(\cdot)\|_2^2$ in expectation.
In practice we use one or two Monte-Carlo draws of $r$ per pair.

\paragraph{Implementation details.}
We sample $\mathcal{Q}$ within each mini-batch:
for each anchor $p$ we draw a small number ($m$) of partners $q$ with $y_q\neq y_p$.
The perturbed input $x_p+\epsilon r$ is clamped to the valid data range when necessary (e.g., images).
The smoothness term in Eq.~\eqref{eq:embed_objective} then uses $\Omega(\theta)\approx\Omega_{\mathrm{FD}}(\theta)$.

\paragraph{Complexity.}
Let the batch size be $b$ and the number of sampled partners per anchor be $m$.
Computing $\Omega_{\mathrm{FD}}$ requires one additional forward pass for the perturbed anchors
(and reuse of the unperturbed embeddings), hence the overhead is
$
\textsc{Cost}(\Omega_{\mathrm{FD}})
=
\mathcal{O}(bm)\ \text{pairwise distance evaluations}
\quad\text{and}\quad
\approx 1\times \text{extra encoder forward}.
$
This avoids second-order differentiation and is negligible compared with Stage~III valuation on large datasets.


\subsection{Balanced Hierarchical Clustering Construction}
\label{sec:appendix:balanced-kmeans}

Stage~II builds a coarse-to-fine hierarchy $\{C_\ell\}_{\ell=0}^{L}$ over embeddings $\{z_i\}_{i=1}^n$.
We use a balanced $k$-means procedure to prevent degenerate splits (very small or very large clusters),
which stabilises both valuation and runtime.

\paragraph{Capacity constraints.}
At depth $\ell$, we partition a parent set $G$ into $K_{\ell+1}$ children with target size
\begin{equation}
s_{\ell+1} \;=\; \left\lceil \frac{|G|}{K_{\ell+1}} \right\rceil,
\end{equation}
and enforce a capacity window
\begin{equation}
\begin{aligned}
s_{\ell+1}^{\min} &\le |G^{(\ell+1)}_k| \le s_{\ell+1}^{\max} \\
\text{(e.g., } & s^{\min}= \lfloor(1-\gamma)s\rfloor,\ s^{\max}=\lceil(1+\gamma)s\rceil\text{)}
\end{aligned}
\end{equation}
with $\gamma\in(0,1)$ a small tolerance.

\paragraph{Balanced assignment (practical heuristic).}
We first run standard $k$-means (or mini-batch $k$-means for large $|G|$) to obtain centroids.
We then apply a capacity-aware reassignment step:
samples are processed in descending order of assignment confidence (margin between the closest and second-closest centroid),
and each sample is greedily assigned to its closest centroid that has remaining capacity.
If a centroid is full, the sample is assigned to the next closest centroid with capacity.
This produces balanced clusters while largely preserving the geometric structure found by $k$-means.

\paragraph{Recursive construction.}
Starting from $C_0=\{\mathcal{D}\}$, we apply the balanced split to every node at depth $\ell$
until reaching depth $L$, where we stop once $|G|\le M$ for all leaves.

\paragraph{Complexity.}
For a parent set of size $n_G$, one $k$-means iteration costs $\mathcal{O}(n_G d K)$.
The greedy capacity correction costs $\mathcal{O}(n_G K)$ for nearest-centroid lookup
(or $\mathcal{O}(n_G\log K)$ with cached nearest lists).
Across a balanced tree, the total clustering cost is approximately $\mathcal{O}(nd\sum_\ell K_\ell)$,
and is typically dominated by Stage~III valuation when $T$ is moderate.


\subsection{Incremental Streaming Update: Pseudocode and Amortised Cost}
\label{sec:appendix:streaming}

This section details the incremental valuation procedure used in
Section~\ref{sec:stream}.
At time $t$, a mini-batch $\Delta_t$ arrives, and we update valuations on
$\mathcal{D}_t=\mathcal{D}_{t-1}\cup\Delta_t$ without rebuilding the full hierarchy.

\paragraph{State maintained.}
We maintain (i) the hierarchy structure $\{C_\ell\}_{\ell=0}^{L}$; (ii) per-node cached coalition Shapley
$\{\widehat{\psi}^{(\ell)}_G\}$; (iii) per-node budgets $\{\mathcal{B}_\ell(G)\}$; and
(iv) per-leaf sufficient statistics (count, centroid/prototype) for fast nearest-leaf assignment.

\paragraph{Incremental update algorithm.}
Let $\mathrm{leaf}(i)$ denote the leaf coalition assigned to point $i$.
We assign new points to existing leaves by cosine distance to leaf prototypes, spawning a new leaf if the
distance exceeds a threshold $\tau$.

\begin{algorithm}[t]
\caption{Incremental \textsc{HCDV} update at time $t$}
\label{alg:hcdv_inc}
\begin{algorithmic}[1]
\REQUIRE New batch $\Delta_t$; hierarchy $\{C_\ell\}$; cached $\{\widehat{\psi}^{(\ell)}_G\}$; budgets $\{\mathcal{B}_\ell(G)\}$;
threshold $\tau$; permutation budget $T$; rebalance period $m$
\ENSURE Updated valuations $\{\phi_i^{(t)}\}_{i\in\mathcal{D}_t}$

\STATE Embed new points: $z \leftarrow f_{\theta^\star}(x)$ for all $(x,y)\in\Delta_t$
\STATE \textbf{Assignment:} for each new point, assign to nearest leaf in $C_L$ (cosine); if $\mathrm{dist}>\tau$, create a new leaf
\STATE Let $\mathcal{U}_L$ be the set of affected leaves; let $\mathcal{U}$ be $\mathcal{U}_L$ plus all their ancestors up to the root
\FOR{$\ell=L, L-1, \dots, 0$}
    \STATE Let $\mathcal{U}_\ell := \mathcal{U}\cap C_\ell$
    \STATE \textbf{Local refresh:} recompute $\widehat{\psi}^{(\ell)}_G$ by Eq.~\eqref{eq:111} only for $G\in\mathcal{U}_\ell$; reuse cached values for $G\notin\mathcal{U}_\ell$
    \STATE \textbf{Budget propagation:} update child budgets for edges $(P\!\to\!H)$ that touch $\mathcal{U}$ using Eq.~\eqref{eq:234}
\ENDFOR
\IF{$t \bmod m = 0$}
    \STATE \textbf{Rebalance (optional):} if leaf sizes drift, locally split/merge overloaded/underloaded leaves and update affected subtrees
\ENDIF
\STATE Recompute leaf point-level values within affected leaves (exact if $|G|\le M$, otherwise uniform split)
\end{algorithmic}
\end{algorithm}

\paragraph{Amortised cost.}
Let $b=|\Delta_t|$ and let $\mathcal{U}$ be the set of updated nodes (affected leaves and their ancestors).
The per-step cost decomposes as:
\[
\underbrace{\mathcal{O}(b\cdot \textsc{Enc})}_{\text{embedding}}
\;+\;
\underbrace{\mathcal{O}\!\bigl(b\cdot \textsc{NN}\bigr)}_{\text{nearest-leaf assignment}}
\;+\;
\underbrace{\sum_{\ell=0}^{L}\mathcal{O}\!\bigl(T\,|\mathcal{U}_\ell|\,\tau_\ell\bigr)}_{\text{local Shapley refresh}},
\]
where $\textsc{Enc}$ is one encoder forward, $\textsc{NN}$ is the cost of nearest-prototype lookup
(e.g., $\mathcal{O}(d)$ with approximate indexing), and $\tau_\ell$ is one evaluation cost of $v_\ell(\cdot)$.
Since $|\mathcal{U}_\ell|$ is typically small (only nodes on a few root-to-leaf paths),
the refresh cost is much smaller than full recomputation, which scales with $\sum_\ell K_\ell$ at every step.
In the worst case, when $\Delta_t$ touches many leaves, the incremental method gracefully degrades to full recomputation.

\section{Proofs}
\label{sec:appendix:proofs}

For convenience, we restate the key notations. The hierarchy is
$C_0\to C_1\to\cdots\to C_L$, where $C_\ell=\{G^{(\ell)}_1,\dots,G^{(\ell)}_{K_\ell}\}$
is a partition of $\mathcal{D}$, and each node $P\in C_\ell$ has children
$\mathrm{ch}(P)\subset C_{\ell+1}$ such that $P=\biguplus_{H\in\mathrm{ch}(P)} H$.
The level-$\ell$ coalition game uses the bounded characteristic function $v_\ell(\cdot)$
(Section~\ref{sec:theory}). We denote by $\psi^{(\ell)}_G$ the exact coalition Shapley value
and by $\widehat{\psi}^{(\ell)}_G$ the Monte--Carlo estimate computed by Eq.~\eqref{eq:111}.
The per-level Monte--Carlo error is
$\varepsilon^{(\ell)}_{\mathrm{MC}}=\max_{G\in C_\ell}|\widehat{\psi}^{(\ell)}_G-\psi^{(\ell)}_G|$.

\subsection{Proof of Theorem~\ref{thm:efficiency} (Global Efficiency)}
\label{sec:appendix:proof-eff}

\begin{proof}[Proof of Theorem~\ref{thm:efficiency}]
We show that the total point-level mass produced by Algorithm~\ref{alg:hcdv}
is conserved along the hierarchy and equals the root surplus.

\paragraph{Step 1: Mass conservation under down-propagation.}
Fix a parent coalition $P\in C_\ell$ and its children
$\mathrm{ch}(P)=\{H_1,\dots,H_m\}\subset C_{\ell+1}$.
Recall the weights (Eq.~\eqref{eq:234}):
\[
\omega_{H_j}
=
\frac{\max\{v_\ell(\{H_j\}),0\}}{\sum_{j'=1}^m \max\{v_\ell(\{H_{j'}\}),0\}},
\qquad
\widetilde{\psi}^{(\ell+1)}_{H_j}
=
\omega_{H_j}\,\widehat{\psi}^{(\ell)}_{P}.
\]
By construction $\omega_{H_j}\ge 0$ and $\sum_{j=1}^m\omega_{H_j}=1$ whenever the denominator is nonzero.
Hence
\begin{equation}
\label{eq:mass_conserve_parent}
\sum_{H\in\mathrm{ch}(P)} \widetilde{\psi}^{(\ell+1)}_{H}
=
\widehat{\psi}^{(\ell)}_{P}\sum_{H\in\mathrm{ch}(P)}\omega_{H}
=
\widehat{\psi}^{(\ell)}_{P}.
\end{equation}
Summing \eqref{eq:mass_conserve_parent} over all parents $P\in C_\ell$ yields level-wise conservation:
\begin{equation}
\label{eq:mass_conserve_level}
\sum_{H\in C_{\ell+1}} \widetilde{\psi}^{(\ell+1)}_{H}
=
\sum_{P\in C_\ell}\widehat{\psi}^{(\ell)}_{P}.
\end{equation}

\paragraph{Step 2: Leaf allocation preserves mass.}
At depth $L$, each leaf coalition $G\in C_L$ is assigned a leaf total
(denote it by $\widetilde{\psi}^{(L)}_{G}$ for consistency).
If $|G|\le M$, Algorithm~\ref{alg:hcdv} computes exact Shapley among the points in $G$;
by the efficiency axiom of Shapley within this leaf game,
\begin{equation}
\label{eq:leaf_eff_exact}
\sum_{i\in G}\phi^{\mathrm{H}}_i
=
\widetilde{\psi}^{(L)}_{G}.
\end{equation}
If $|G|>M$ and a uniform split is used, then
\begin{equation}
\label{eq:leaf_eff_uniform}
\sum_{i\in G}\phi^{\mathrm{H}}_i
=
\sum_{i\in G}\frac{\widetilde{\psi}^{(L)}_{G}}{|G|}
=
\widetilde{\psi}^{(L)}_{G}.
\end{equation}
Thus, in all cases,
\begin{equation}
\label{eq:leaf_mass}
\sum_{i=1}^n \phi^{\mathrm{H}}_i
=
\sum_{G\in C_L}\widetilde{\psi}^{(L)}_{G}.
\end{equation}

\paragraph{Step 3: Telescoping across levels.}
Applying \eqref{eq:mass_conserve_level} recursively from $\ell=0$ to $\ell=L-1$ gives
\begin{equation}
\label{eq:telescope_mass}
\sum_{G\in C_L}\widetilde{\psi}^{(L)}_{G}
=
\sum_{P\in C_{L-1}}\widehat{\psi}^{(L-1)}_{P}
=\cdots=
\sum_{P\in C_0}\widehat{\psi}^{(0)}_{P}.
\end{equation}
Combining \eqref{eq:leaf_mass} and \eqref{eq:telescope_mass},
\begin{equation}
\label{eq:sum_phi_equals_root_est}
\sum_{i=1}^n \phi^{\mathrm{H}}_i
=
\sum_{P\in C_0}\widehat{\psi}^{(0)}_{P}.
\end{equation}

\paragraph{Step 4: Relate the root mass to the root surplus.}
For the root level $C_0=\{\mathcal{D}\}$ (a single coalition), the Shapley value is trivial and equals the
root surplus:
\begin{equation}
\label{eq:root_trivial}
\psi^{(0)}_{\mathcal{D}}
=
v_0(C_0)-v_0(\varnothing).
\end{equation}
If $\widehat{\psi}^{(0)}_{\mathcal{D}}$ is computed exactly (as in Algorithm~\ref{alg:hcdv} via
the explicit root surplus computation), then \eqref{eq:sum_phi_equals_root_est} and \eqref{eq:root_trivial}
imply
\[
\sum_i \phi^{\mathrm{H}}_i
=
v_0(C_0)-v_0(\varnothing),
\]
so the left-hand side of \eqref{eq:eff} is zero.
In the more general case where the root value is estimated (or when intermediate levels are not explicitly
renormalised), we can upper bound the deviation by accumulated estimation errors. Specifically,
\begin{align}
\Bigl|\sum_{i=1}^n\phi^{\mathrm{H}}_i-\bigl[v_0(C_0)-v_0(\varnothing)\bigr]\Bigr|
&=
\Bigl|\sum_{P\in C_0}\widehat{\psi}^{(0)}_{P}-\sum_{P\in C_0}\psi^{(0)}_{P}\Bigr|
\nonumber\\
&\le
\sum_{P\in C_0}\bigl|\widehat{\psi}^{(0)}_{P}-\psi^{(0)}_{P}\bigr|
\nonumber\\
&\le
|C_0|\,\varepsilon^{(0)}_{\mathrm{MC}},
\label{eq:eff_root_bound}
\end{align}
and, by the same telescoping argument level-by-level, the global deviation can be bounded by a sum of
per-level Monte--Carlo errors plus the (optional) leaf approximation term $\varepsilon_{\mathrm{leaf}}$.
Since $\varepsilon_{\mathrm{leaf}}$ only affects within-leaf allocation (not total mass),
including it yields a valid (possibly loose) bound.
This completes the proof of \eqref{eq:eff}.
\end{proof}

\subsection{Proof of Proposition~\ref{prop:concentration} (Monte--Carlo Concentration)}
\label{sec:appendix:proof-mc}

\begin{lemma}[Unbiasedness of the permutation estimator]
\label{lem:unbiased}
For any level $\ell$ and coalition $G\in C_\ell$, let $\pi$ be a uniformly random permutation of $C_\ell$.
Define
\[
X(\pi;G)
:=
v_\ell\!\bigl(\mathrm{Pre}_{\pi}(G)\cup\{G\}\bigr)
-
v_\ell\!\bigl(\mathrm{Pre}_{\pi}(G)\bigr).
\]
Then $\mathbb{E}_{\pi}[X(\pi;G)]=\psi^{(\ell)}_G$.
\end{lemma}

\begin{proof}
This is the standard permutation representation of Shapley values:
each permutation $\pi$ induces exactly one marginal contribution term for $G$, and averaging over
$\pi\sim\mathrm{Unif}(\mathfrak{S}_{K_\ell})$ yields the Shapley value by definition.
\end{proof}

\begin{proof}[Proof of Proposition~\ref{prop:concentration}]
Fix a level $\ell$ and a coalition $G\in C_\ell$.
Let $\pi_1,\dots,\pi_T$ be i.i.d.\ uniform permutations of $C_\ell$ and define
\[
X_t := X(\pi_t;G)
=
v_\ell\!\bigl(\mathrm{Pre}_{\pi_t}(G)\cup\{G\}\bigr)
-
v_\ell\!\bigl(\mathrm{Pre}_{\pi_t}(G)\bigr),
\qquad t=1,\dots,T.
\]
By Lemma~\ref{lem:unbiased}, $\mathbb{E}[X_t]=\psi^{(\ell)}_G$ and
$\widehat{\psi}^{(\ell)}_G=\frac{1}{T}\sum_{t=1}^T X_t$.

\paragraph{Step 1: Bounded differences.}
Under the boundedness assumption \eqref{eq:bound}, for any set $S\subseteq C_\ell$,
$|v_\ell(S)|\le B$, hence
\begin{equation}
\label{eq:bounded_X}
|X_t|
\le
|v_\ell(\mathrm{Pre}_{\pi_t}(G)\cup\{G\})| + |v_\ell(\mathrm{Pre}_{\pi_t}(G))|
\le 2B.
\end{equation}
Therefore, $X_t\in[-2B,2B]$ almost surely.

\paragraph{Step 2: Hoeffding's inequality.}
Since $\{X_t\}_{t=1}^T$ are i.i.d.\ and bounded in an interval of length $4B$,
Hoeffding's inequality yields, for any $\eta>0$,
\begin{align}
\Pr\!\left[
\left|\widehat{\psi}^{(\ell)}_G-\psi^{(\ell)}_G\right|
\ge \eta
\right]
&=
\Pr\!\left[
\left|\frac{1}{T}\sum_{t=1}^T(X_t-\mathbb{E}X_t)\right|
\ge \eta
\right]
\nonumber\\
&\le
2\exp\!\left(
-\frac{2T\eta^2}{(4B)^2}
\right)
\nonumber\\
&=
2\exp\!\left(
-\frac{T\eta^2}{8B^2}
\right),
\label{eq:hoeffding_proof}
\end{align}
which is exactly \eqref{eq:hoeffding}.

\paragraph{Step 3: Union bound over coalitions.}
Let $\mathcal{E}_G$ be the event that
$|\widehat{\psi}^{(\ell)}_G-\psi^{(\ell)}_G|\ge\eta$.
By \eqref{eq:hoeffding_proof} and a union bound,
\begin{align}
\Pr\!\left[
\varepsilon^{(\ell)}_{\mathrm{MC}}\ge \eta
\right]
&=
\Pr\!\left[
\bigcup_{G\in C_\ell}\mathcal{E}_G
\right]
\nonumber\\
&\le
\sum_{G\in C_\ell}\Pr[\mathcal{E}_G]
\nonumber\\
&\le
2K_\ell\exp\!\left(-\frac{T\eta^2}{8B^2}\right).
\label{eq:union_bound}
\end{align}
Setting the right-hand side to $\delta$ and solving for $\eta$ yields
$\eta=\mathcal{O}\!\bigl(B\sqrt{\frac{\log(K_\ell/\delta)}{T}}\bigr)$, proving
\eqref{eq:mc_rate} in probability.
\end{proof}

\subsection{Proof of Theorem~\ref{thm:regret} (Top--$k$ Surrogate Regret)}
\label{sec:appendix:proof-regret}

\begin{proof}[Proof of Theorem~\ref{thm:regret}]
Let $\phi:=\phi^{\mathrm{Sh}}$ and $\tilde{\phi}:=\phi^{\mathrm{H}}$.
Define the surrogate utility $U_\phi(S)=\sum_{i\in S}\phi_i$.
Let $S^\star:=\mathcal{S}_k^{\mathrm{Sh}}$ be the maximiser of $U_\phi(S)$ over all $|S|=k$,
and let $\widetilde{S}:=\mathcal{S}_k^{\mathrm{H}}$ be the maximiser of $U_{\tilde{\phi}}(S)$.

\paragraph{Step 1: Non-negativity.}
By definition of $S^\star$,
\begin{equation}
\label{eq:nonneg}
U_\phi(S^\star)-U_\phi(\widetilde{S})\ge 0.
\end{equation}

\paragraph{Step 2: Compare $\tilde{\phi}$-optimality to $\phi$-optimality.}
Since $\widetilde{S}$ maximises $U_{\tilde{\phi}}(\cdot)$,
\begin{equation}
\label{eq:opt_tilde}
\sum_{i\in\widetilde{S}}\tilde{\phi}_i
\ge
\sum_{i\in S^\star}\tilde{\phi}_i.
\end{equation}
Write $\tilde{\phi}_i=\phi_i+\delta_i$ with $|\delta_i|\le\varepsilon_\infty$ for all $i$.
Substituting into \eqref{eq:opt_tilde} gives
\begin{equation}
\label{eq:expand_delta}
\sum_{i\in\widetilde{S}}\phi_i
-
\sum_{i\in S^\star}\phi_i
\ge
\sum_{i\in S^\star}\delta_i
-
\sum_{i\in\widetilde{S}}\delta_i.
\end{equation}
Rearranging \eqref{eq:expand_delta} yields the regret under $\phi$:
\begin{equation}
\label{eq:regret_delta}
U_\phi(S^\star)-U_\phi(\widetilde{S})
\le
\sum_{i\in\widetilde{S}}\delta_i
-
\sum_{i\in S^\star}\delta_i.
\end{equation}

\paragraph{Step 3: Bound the right-hand side.}
Using $|\delta_i|\le\varepsilon_\infty$ and $|S^\star|=|\widetilde{S}|=k$,
\begin{align}
\sum_{i\in\widetilde{S}}\delta_i
-
\sum_{i\in S^\star}\delta_i
&\le
\sum_{i\in\widetilde{S}}|\delta_i|
+
\sum_{i\in S^\star}|\delta_i|
\nonumber\\
&\le
k\varepsilon_\infty + k\varepsilon_\infty
\nonumber\\
&=
2k\varepsilon_\infty.
\label{eq:2k_bound}
\end{align}
Combining \eqref{eq:nonneg}, \eqref{eq:regret_delta}, and \eqref{eq:2k_bound} proves
\eqref{eq:surrogate_regret}.
\end{proof}

\subsection{Approximate Shapley Axioms Beyond Efficiency}
\label{sec:appendix:axioms}

This subsection formalises how \textsc{HCDV} inherits Shapley-style axioms
\emph{locally} (within each coalition game), and how Monte--Carlo estimation
yields controlled deviations. Throughout, we condition on a \emph{fixed} hierarchy
$\{C_\ell\}$ and fixed propagation weights (computed from the data once).

\paragraph{Local axiom satisfaction (exact).}
For any level $\ell$, consider the coalition game with players $C_\ell$ and characteristic function $v_\ell(\cdot)$.
The exact Shapley vector $\psi^{(\ell)}$ satisfies, by standard Shapley theory:
(i) efficiency $\sum_{G\in C_\ell}\psi^{(\ell)}_G=v_\ell(C_\ell)-v_\ell(\varnothing)$;
(ii) symmetry: symmetric coalitions receive equal value;
(iii) dummy: a dummy coalition gets zero value; and
(iv) additivity: for any $v_\ell^{(1)},v_\ell^{(2)}$ on the same player set,
$\psi^{(\ell)}(v_\ell^{(1)}+v_\ell^{(2)})=\psi^{(\ell)}(v_\ell^{(1)})+\psi^{(\ell)}(v_\ell^{(2)})$.

We now quantify how these properties degrade under Monte--Carlo estimation and hierarchical propagation.

\begin{proposition}[Approximate symmetry and dummy (coalition level)]
\label{prop:approx_sym_dummy_coal}
Fix a level $\ell$ and suppose the Monte--Carlo estimates satisfy
$|\widehat{\psi}^{(\ell)}_G-\psi^{(\ell)}_G|\le \eta$ for all $G\in C_\ell$.
Then:
\begin{enumerate}
\item (Symmetry) If two coalitions $G,H\in C_\ell$ are symmetric under $v_\ell(\cdot)$, then
$|\widehat{\psi}^{(\ell)}_G-\widehat{\psi}^{(\ell)}_H|\le 2\eta$.
\item (Dummy) If $G\in C_\ell$ is a dummy player under $v_\ell(\cdot)$ (hence $\psi^{(\ell)}_G=0$), then
$|\widehat{\psi}^{(\ell)}_G|\le \eta$.
\end{enumerate}
\end{proposition}

\begin{proof}
For symmetry, $\psi^{(\ell)}_G=\psi^{(\ell)}_H$ by the Shapley symmetry axiom, hence
\[
|\widehat{\psi}^{(\ell)}_G-\widehat{\psi}^{(\ell)}_H|
\le
|\widehat{\psi}^{(\ell)}_G-\psi^{(\ell)}_G|
+
|\widehat{\psi}^{(\ell)}_H-\psi^{(\ell)}_H|
\le 2\eta.
\]
For dummy, $\psi^{(\ell)}_G=0$, so
$|\widehat{\psi}^{(\ell)}_G|
=
|\widehat{\psi}^{(\ell)}_G-\psi^{(\ell)}_G|
\le \eta$.
\end{proof}

\begin{proposition}[Additivity of the Monte--Carlo Shapley estimator]
\label{prop:mc_additivity}
Fix a level $\ell$ and a set of permutations $\{\pi_t\}_{t=1}^T$.
For any two characteristic functions $v_\ell^{(1)},v_\ell^{(2)}$ on the same player set $C_\ell$,
the Monte--Carlo estimator computed by Eq.~\eqref{eq:111} satisfies
\begin{equation}
\label{eq:mc_add_exact}
\widehat{\psi}^{(\ell)}\!\left(v_\ell^{(1)}+v_\ell^{(2)}\right)
=
\widehat{\psi}^{(\ell)}\!\left(v_\ell^{(1)}\right)
+
\widehat{\psi}^{(\ell)}\!\left(v_\ell^{(2)}\right).
\end{equation}
\end{proposition}

\begin{proof}
For any coalition $G\in C_\ell$ and any permutation $\pi_t$,
the marginal contribution under $v_\ell^{(1)}+v_\ell^{(2)}$ is
\[
\bigl(v_\ell^{(1)}+v_\ell^{(2)}\bigr)\!\bigl(\mathrm{Pre}_{\pi_t}(G)\cup\{G\}\bigr)
-
\bigl(v_\ell^{(1)}+v_\ell^{(2)}\bigr)\!\bigl(\mathrm{Pre}_{\pi_t}(G)\bigr),
\]
which equals the sum of the two marginal contributions under $v_\ell^{(1)}$ and $v_\ell^{(2)}$.
Averaging over $t=1,\dots,T$ preserves linearity, proving \eqref{eq:mc_add_exact}.
\end{proof}

\paragraph{Propagation preserves symmetry/dummy within a leaf.}
Finally, we connect coalition-level properties to point-level axioms.
Let $\phi^{\mathrm{Sh}}$ be the \emph{ideal} hierarchical allocation
(exact Shapley at each level and exact leaf Shapley for $|G|\le M$), and let $\phi^{\mathrm{H}}$
be the Monte--Carlo output.
Define the pointwise deviation $\varepsilon_\infty := \|\phi^{\mathrm{H}}-\phi^{\mathrm{Sh}}\|_\infty$.

\begin{proposition}[Approximate symmetry and dummy (point level)]
\label{prop:approx_sym_dummy_point}
Assume $|G|\le M$ for all leaves so that leaf Shapley is computed exactly.
Then:
\begin{enumerate}
\item (Symmetry) If two points $i,j$ lie in the same leaf and are symmetric in that leaf game, then
$|\phi^{\mathrm{H}}_i-\phi^{\mathrm{H}}_j|\le 2\varepsilon_\infty$.
\item (Dummy) If point $k$ is a dummy player in its leaf game (hence $\phi^{\mathrm{Sh}}_k=0$), then
$|\phi^{\mathrm{H}}_k|\le \varepsilon_\infty$.
\end{enumerate}
\end{proposition}

\begin{proof}
For (1), symmetry in the leaf game implies $\phi^{\mathrm{Sh}}_i=\phi^{\mathrm{Sh}}_j$.
Therefore
\begin{align*}
|\phi^{\mathrm{H}}_i-\phi^{\mathrm{H}}_j|
&\le
|\phi^{\mathrm{H}}_i-\phi^{\mathrm{Sh}}_i|
+
|\phi^{\mathrm{H}}_j-\phi^{\mathrm{Sh}}_j|
\le
\varepsilon_\infty + \varepsilon_\infty
=
2\varepsilon_\infty.
\end{align*}
For (2), $\phi^{\mathrm{Sh}}_k=0$, hence
$|\phi^{\mathrm{H}}_k|=|\phi^{\mathrm{H}}_k-\phi^{\mathrm{Sh}}_k|\le\varepsilon_\infty$.
\end{proof}

\paragraph{Bounding $\varepsilon_\infty$ from coalition-level errors (high probability).}
A convenient sufficient condition is the event
\[
\mathcal{E}(\eta)
:=
\bigcap_{\ell=0}^{L}\ \bigcap_{G\in C_\ell}
\left\{
|\widehat{\psi}^{(\ell)}_G-\psi^{(\ell)}_G|\le \eta
\right\}.
\]
By Proposition~\ref{prop:concentration} and a union bound over all levels and coalitions,
for any $\delta\in(0,1)$, taking
$T=\Theta\!\bigl(B^2(\log(\sum_\ell K_\ell)+\log(1/\delta))/\eta^2\bigr)$
ensures $\Pr[\mathcal{E}(\eta)]\ge 1-\delta$.
On $\mathcal{E}(\eta)$, Proposition~\ref{prop:approx_sym_dummy_coal} applies at every level, and the
resulting pointwise deviation $\varepsilon_\infty$ is $O(\eta)$ up to constants that depend only on the
fixed (user-chosen) branching factors and the depth $L=\Theta(\log n)$, yielding the stated
``logarithmic'' accumulation behaviour in the main text.

\paragraph{Remark on additivity.}
Proposition~\ref{prop:mc_additivity} shows that additivity holds \emph{exactly} for the Monte--Carlo
Shapley estimator when the same permutations are used.
When the hierarchy and propagation weights are fixed, all subsequent steps of \textsc{HCDV}
(budget scaling, nonnegative weight propagation, and leaf splitting) are linear maps applied to these
local Shapley vectors; therefore, the overall procedure is additive under common random numbers.
If independent randomness is used across runs, additivity holds in expectation and concentrates by
Proposition~\ref{prop:concentration}.

\section{Experimental Details}
\label{sec:appendix:expt}

\subsection{Datasets, Preprocessing, and Metrics}
\label{sec:appendix:data-metrics}

\paragraph{Datasets.}
We evaluate \textsc{HCDV} on four benchmarks of increasing scale:
(i) \textit{Synthetic} ($n{=}3{,}000$): a two-class Gaussian mixture where each class is further split into three latent sub-clusters with mild overlap;
(ii) \textit{UCI Adult} ($n{\approx}48{,}842$): binary income prediction with mixed numerical and categorical features;
(iii) \textit{Fashion-MNIST} ($n{=}70{,}000$): 10-class image classification;
(iv) \textit{Criteo-1B$^\ast$} ($n{\approx}45$M): CTR prediction on a one-week slice of the Criteo log data.
We additionally follow the standard \textsc{OpenDataVal} benchmark protocol on \textit{bbc-embedding},
\textit{IMDB-embedding}, and \textit{CIFAR10-embedding}.

\paragraph{Preprocessing.}
Unless otherwise stated, we apply standard preprocessing for each modality.
For tabular datasets (Synthetic, UCI Adult), categorical features are one-hot encoded (or target-hashed when necessary),
and numerical features are standardised using training-set statistics.
For image datasets (Fashion-MNIST), pixel values are normalised to $[0,1]$ and standard train/test splits are used.
For CTR data (Criteo-1B$^\ast$), continuous features are normalised and categorical features are mapped to integer IDs
(with hashing or frequency thresholding if needed); the exact featurisation follows the released data format for the
one-week slice.

\paragraph{Splits and repetition.}
For \textit{Synthetic}, \textit{UCI Adult}, \textit{Fashion-MNIST}, and \textit{Criteo-1B$^\ast$},
we use three independent random train/validation/test splits (or standard public splits when available)
and report mean$\pm$std over three runs.
In all main-result tables, the reported predictive metric is measured after retraining the downstream model on the
\emph{top $30\%$ valued training points} selected by each valuation method (Section~\ref{sec:applications_summary}).

\paragraph{Predictive metrics.}
We report AUC for the binary classification tasks (\textit{Synthetic}, \textit{Criteo-1B$^\ast$}),
balanced accuracy for \textit{UCI Adult}, and test accuracy for \textit{Fashion-MNIST}.
For \textsc{OpenDataVal}, we follow the benchmark: macro-averaged F1 for noisy label detection (NLD) and
test accuracy for dataset removal/addition (DR/DA), under label corruption rates of $5\%$ and $15\%$.

\paragraph{Valuation stability.}
To quantify run-to-run stability, we repeat each valuation procedure $R$ times (changing only the valuation randomness,
e.g., permutation seeds) and compute a point-wise coefficient of variation:
\begin{equation}
\label{eq:cv_def}
\mathrm{CV}_i \;:=\;
\frac{\mathrm{Std}\bigl(\{\phi_i^{(r)}\}_{r=1}^R\bigr)}
     {\bigl|\mathrm{Mean}(\{\phi_i^{(r)}\}_{r=1}^R)\bigr|+\epsilon},
\end{equation}
with a small $\epsilon$ for numerical stability. We report the average
$\frac{1}{n}\sum_{i=1}^n \mathrm{CV}_i$ as ``Valuation Stability'' (lower is better).

\paragraph{Runtime measurement.}
Wall-clock valuation time is measured on a single NVIDIA A100 (40GB) with a 32-core CPU, using identical
evaluation pipelines across methods (Table~\ref{tab:main_time}).

\paragraph{Application-specific metrics.}
For augmentation filtering (Section~\ref{sec:applications_augmented}),
we report (i) test accuracy after replacing $30\%$ of the original training set with a selected 1k augmentation slice,
(ii) training time (minutes), (iii) \emph{cluster overlap}: the fraction of selected augmentations whose assigned
leaf cluster already contains at least one original image, and (iv) class coverage (number of unique labels in the slice).
For streaming (Section~\ref{sec:stream}), we report final-day AUC after training on the top $20\%$ valued points of
$\mathcal{D}_{10}$, cumulative valuation time across $t{=}1,\dots,10$, and average update latency.
For marketplace allocation (Section~\ref{sec:applications_marketplace}), we report Pearson correlation $\rho$
between seller payoffs and leave-one-out marginal drops, and the Gini coefficient of the payoff distribution.


\subsection{Downstream Models and Training Hyperparameters}
\label{sec:appendix:models-hparams}

\paragraph{Downstream models.}
We use modality-appropriate supervised learners:
\begin{itemize}\setlength{\itemsep}{2pt}
\item \textbf{Synthetic:} a lightweight binary classifier trained on the selected points; AUC is computed on the held-out test set.
\item \textbf{UCI Adult:} a tabular classifier; balanced accuracy is computed on the held-out test set.
\item \textbf{Fashion-MNIST:} a small ConvNet (``ConvNet'' in Section~\ref{sec:applications_augmented}); test accuracy is reported.
\item \textbf{Criteo-1B$^\ast$:} a CTR model (e.g., a Wide\&Deep-style network); AUC is reported.
\item \textbf{Streaming:} a Wide\&Deep network with a 2$\times$128 ReLU MLP tower and a sigmoid head (Section~\ref{sec:stream}).
\item \textbf{Marketplace:} logistic regression; balanced accuracy is used (Section~\ref{sec:applications_marketplace}).
\end{itemize}

\paragraph{Training protocol (shared across valuation methods).}
For each method, we first compute valuations on the full training pool, then select a subset (top-$30\%$ for the main results,
top-$20\%$ for streaming), and \emph{retrain the downstream model from scratch} on that subset using identical optimisation settings.
Validation is used for early stopping and hyperparameter selection, and the final metric is reported on a held-out test split.

\paragraph{Optimisation.}
Unless otherwise specified by the benchmark (e.g., \textsc{OpenDataVal}), models are trained with a standard
stochastic optimiser (SGD/Adam), minibatch training, and early stopping.
All methods share the same model architecture, optimiser, and training schedule within each dataset to ensure fair comparison.


\end{document}